\documentclass{article}
\usepackage[final, nonatbib]{neurips_2024}

\usepackage[utf8]{inputenc} 
\usepackage[T1]{fontenc}    
\usepackage{hyperref}       
\usepackage{url}            
\usepackage{booktabs}       
\usepackage{amsfonts}       
\usepackage{nicefrac}       
\usepackage{microtype}      
\usepackage{xcolor}         
\usepackage{enumitem}

\usepackage{amsmath}
\usepackage{amssymb}
\usepackage{mathtools}
\usepackage{amsthm}
\usepackage{mysymbol}
\usepackage{bm}
\usepackage{algorithm}
\usepackage{algorithmic}
\usepackage{wrapfig}

\theoremstyle{plain}
\newtheorem{theorem}{Theorem}

\newtheorem{lemma}{Lemma}

\theoremstyle{definition}
\newtheorem{definition}{Definition}
\newtheorem{assumption}{Assumption}
\theoremstyle{remark}
\newtheorem{remark}{Remark}

\title{Learning General Parameterized Policies for Infinite Horizon Average Reward Constrained MDPs via Primal-Dual Policy Gradient Algorithm}

\author{
  Qinbo Bai\\
  Purdue University\\
  West Lafayette, IN 47906 \\
  \texttt{bai113@purdue.edu} \\
  \And
  Washim Uddin Mondal \\
  Indian Institute of Technology Kanpur\\
  Kanpur,  UP, India 208016\\
  \texttt{wmondal@iitk.ac.in} \\
  \And
  Vaneet Aggarwal \\
  Purdue University\\
  West Lafayette, IN 47906 \\
  \texttt{vaneet@purdue.edu} \\
}

\begin{document}
\maketitle

\begin{abstract}
  This paper explores the realm of infinite horizon average reward Constrained Markov Decision Processes (CMDPs). To the best of our knowledge, this work is the first to delve into the regret and constraint violation analysis of average reward CMDPs with a general policy parametrization. To address this challenge, we propose a primal dual-based policy gradient algorithm that adeptly manages the constraints while ensuring a low regret guarantee toward achieving a global optimal policy. In particular, our proposed algorithm achieves $\tilde{\mathcal{O}}({T}^{4/5})$ objective regret and $\tilde{\mathcal{O}}({T}^{4/5})$ constraint violation bounds.
\end{abstract}

\section{Introduction}

The framework of Reinforcement Learning (RL) is concerned with a class of problems where an agent learns to yield the maximum cumulative reward in an unknown environment via repeated interaction. RL finds applications in diverse areas, such as wireless communication, transportation, and epidemic control \cite{liu2021cmix,al2019deeppool, ling2023cooperating}. RL problems are mainly categorized into three setups: episodic, infinite horizon discounted reward, and infinite horizon average reward. Among them, the infinite horizon average reward setup is particularly significant for real-world applications. It aligns with most of the practical scenarios and captures their long-term goals. Some applications in real life require the learning procedure to respect the boundaries of certain constraints. In an epidemic control setup, for example, vaccination policies must take the supply shortage (budget constraint) into account. Such restrictive decision-making routines are described by constrained Markov Decision Processes (CMDP) \cite{bai2023achieving, agarwal2022concave,chen2022learning}. Existing papers on CMDPs utilize either a tabular or a linear MDP structure. This work provides the first algorithm for an infinite horizon average reward CMDP with general parametrization and proves its sub-linear regret and constraint violation bounds. 

There are two primary ways to solve a CMDP problem in the infinite horizon average reward setting. The first one, known as the model-based approach, involves constructing estimates of the transition probabilities of the underlying CMDP, which are subsequently utilized to derive policies \cite{chen2022learning, agarwal2022regret,agarwal2022concave}. The caveat of this approach is the large memory requirement to store the estimated parameters, which effectively curtails its applicability to CMDPs with large state spaces. The alternative strategy, known as the model-free approach, either directly estimates the policy function or maintains an estimate of the $Q$ function, which is subsequently used for policy generation \cite{wei2022provably}. Model-free algorithms typically demand lower memory and computational resources than their model-based counterparts. Although the CMDP has been solved in a model-free manner in the tabular \cite{wei2022provably} and linear \cite{ghosh2022achieving} setups, its exploration with the general parameterization is still open and is the goal of this paper. 
 
\begin{table*}[t]
    \centering
    \resizebox{\textwidth}{!}
    {
    \begin{tabular}{|c|c|c|c|c|}
        \hline
        Algorithm & Regret & Violation & Model-free & Setting\\
	   \hline
         & & & &\\[-0.32cm]
        Algorithm 1 in \cite{chen2022learning} & $\tilde{\mathcal{O}}(\sqrt{T})$  & $\tilde{\mathcal{O}}(\sqrt{T})$  & No & Tabular\\
        \hline
        & & & &\\[-0.32cm]
        Algorithm 2 in \cite{chen2022learning} & $\tilde{\mathcal{O}}(T^{2/3})$ & $\tilde{\mathcal{O}}(T^{2/3})$ & No & Tabular\\
        \hline
         & & & &\\[-0.32cm]
        UC-CURL and PS-CURL \cite{agarwal2022concave} & $\tilde{\mathcal{O}}(\sqrt{T})$ & $0$ & No & Tabular\\
	\hline
         & & & &\\[-0.32cm]
        Algorithm 2 in \cite{ghosh2022achieving} & $\tilde{\mathcal{O}}((dT)^{3/4})$ & $\tilde{\mathcal{O}}((dT)^{3/4})$ & No & Linear MDP\\
	\hline
         & & & &\\[-0.32cm]
        Algorithm 3 in \cite{ghosh2022achieving} & $\tilde{\mathcal{O}}(\sqrt{T})$ & $\tilde{\mathcal{O}}(\sqrt{T})$ & No & Linear MDP\\
	\hline
         & & & &\\[-0.32cm]
        Triple-QA \cite{wei2022provably} & $\tilde{\mathcal{O}}(T^{5/6})$ & $0$ & Yes & Tabular\\
	\hline
         & & & &\\[-0.32cm]
        This paper & $\tilde{\mathcal{O}}(T^{\frac{4}{5}})$ & $\tilde{\mathcal{O}}(T^{\frac{4}{5}})$ & Yes & General Parameterization\\
	\hline
	\end{tabular}}
  \vspace{-.1in}
        \caption{\small This table summarizes the different model-based and mode-free state-of-the-art algorithms available in the literature for average reward CMDPs. We note that our proposed algorithm is the first to analyze the regret and constraint violation for average reward CMDP with general parametrization. Here, the parameter $d$ refers to the dimension of the feature map for linear MDPs.}
	\label{table2}
 \vspace{-.3in}
\end{table*}

General parameterization indexes the policies by finite-dimensional parameters (e.g., weights of neural networks) to accommodate large state spaces. The learning is manifested by updating these parameters using policy gradient (PG)-type algorithms. Note that PG algorithms are primarily studied in discounted reward setups. For example, \cite{agarwal2021theory} characterizes the sample complexities of the PG and the Natural PG (NPG) algorithms with softmax and direct parameterization. Similar results for general parameterization are obtained by \cite{ mondal2023improved, mondal2024sample}. 
The regret analysis of a PG algorithm with the general parameterization has been recently performed for an infinite horizon average reward MDP without constraints  \cite{bai2023regret}. Similar regret and constraint violation analysis for the average reward CMDP is still missing in the literature. In this paper, we bridge this gap.

{\bf Challenges and Contribution:} We propose a PG-based algorithm with general parameterized policies for the average reward CMDP and establish its sublinear regret and constraint violation bounds. In particular, assuming the underlying CMDP to be ergodic, we demonstrate that our PG algorithm achieves an average optimality rate of $\tilde{\mathcal{O}}(T^{-\frac{1}{5}})$ and average constraint violation rate of $\tilde{\mathcal{O}}(T^{-\frac{1}{5}})$. Invoking this convergence result, we establish that our algorithm achieves regret and constraint violation bounds of $\tilde{\mathcal{O}}(T^{\frac{4}{5}})$. Apart from providing the first sublinear regret guarantee for the average reward CMDP with general parameterization, our work also improves the state-of-the-art regret guarantee, $\tilde{\mathcal{O}}(T^{5/6})$ in the model-free tabular setup \cite{wei2022provably}.

Despite the availability of sample complexity analysis of PG algorithms with constraints in the discounted reward setup \cite{ding2020natural,bai2023achieving} and PG algorithms without constraint in average reward setup \cite{bai2023regret}, obtaining sublinear regret and constraint violation bounds for their average reward counterpart is challenging. 
\begin{itemize}[left=0pt, labelwidth=0pt, labelsep=0pt, itemindent=2\parindent, itemsep=0pt, parsep=0pt, topsep=0pt, partopsep=0pt]
    \item \cite{ding2020natural,bai2023achieving} solely needs an estimate of the value function $V$ while we additionally need the estimate of the gain function, $J$. 
    \item  \cite{ding2020natural,bai2023achieving} assume access to a simulator to generate unbiased value estimates. In contrast, our algorithm uses a sample trajectory of length $H$ to estimate the values and gains and does not assume the availability of a simulator.
    \item  The first-order convergence analysis (Lemma 6) differs from that in \cite{bai2023regret}. Note that both of these papers use an ascent-like inequality. In \cite{bai2023regret}, this bounds the term $J(\theta_{k+1})-J(\theta_k)$. The final result is obtained by calculating a sum over $k$ which cancels the intermediate terms and leaves us with $J(\theta_K)-J(\theta_1)$. We would like to emphasize that the cancellation of the intermediate terms is crucial to establishing the result. However, a similar effort in our case only leads to a bound of $J_{\mathrm{L}}(\theta_{k+1}, \lambda_k)-J_{\mathrm{L}}(\theta_k, \lambda_k)$. Note that directly performing a sum over this difference does not lead to the cancellation of intermediate terms. We had to take a different route and apply the bounds of the Lagrange multipliers and the estimate of the constraint function to achieve that goal.
    \item After solving the problems mentioned above, we prove $\tilde{\mathcal{O}}(T^{-\frac{1}{5}})$ convergence rate of the Lagrange function. Unfortunately, the strong duality property, which is central to proving convergence results of CMDPs for tabular and softmax policies, does not hold under the general parameterization. As a result, the convergence result for the dual problem does not automatically translate to that for the primal problem, which is a main difference from \cite{bai2023regret}. We overcome this barrier by introducing a novel constraint violation analysis and a series of intermediate results (Lemma \ref{lem.boundness}-\ref{lem.constraint}) that help disentangle the regret and constraint violation rates from the Lagrange convergence. It is important to mention that although the techniques applied are inspired by the \cite{ding2020natural}, those techniques cannot be directly adopted for average reward MDPs. This is primarily because the estimate $\hat{J}_c(\theta_k)$ is biased in the average case. To the best of our knowledge, constraint violation analysis with a biased estimate of the cost value is not available in the literature and is performed for the first time in our paper.
    \item Due to the presence of the Lagrange multiplier, the convergence analysis of a CMDP is much more convoluted than its unconstrained counterpart. The learning rate of the Lagrange update, $\beta$, turns out to be pivotal in determining the growth rate of regret and constraint violation. Low values of $\beta$ push the regret down while simultaneously increasing the constraint violation. Finding the optimal value of $\beta$ that judiciously balances these two competing goals is one of the cornerstones of our analysis. 
\end{itemize}

{\bf Related work for unconstrained average reward RL: } In the absence of constraints, both model-based and model-free tabular setups have been widely studied for infinite horizon average reward MDPs. For example, the model-based algorithms proposed by \cite{agrawal2017optimistic, auer2008near} achieve the optimal regret bound of $\tilde{\mathcal{O}}(\sqrt{T})$. Similarly, the model-free algorithm proposed by \cite{wei2020model} for tabular MDP results in $\tilde{\mathcal{O}}(\sqrt{T})$ regret. Regret analysis for average reward MDP with general parametrization has been recently studied in \cite{bai2023regret}, where a regret bound of $\tilde{\mathcal{O}}({T}^{3/4})$ is derived. 

{\bf Related work for constrained RL:} The constrained reinforcement learning problem has been extensively studied both for infinite horizon discounted reward and episodic MDPs. For example, discounted reward CMDPs have been recently studied in the tabular setup \cite{bai2022achieving}, with both softmax \cite{ding2020natural,xu2021crpo}, and general policy parameterization \cite{ding2020natural,xu2021crpo,bai2023achieving, mondal2024sample}. Moreover, \cite{efroni2020exploration,qiu2020upper,germano2023best} investigated episodic CMDPs in the tabular setting. 

Recently, the infinite horizon average reward CMDPs have been investigated in model-based setups \cite{agarwal2022concave, chen2022learning, agarwal2022regret}, tabular model-free setting \cite{wei2022provably} and linear CMDP setting \cite{ghosh2022achieving}. For model-based CMDP setup, \cite{chen2022learning} proposed a model-based online mirror descent algorithm in the ergodic setting which achieves $\mathcal{\tilde{O}}(\sqrt{T})$ for regret and violation at the same time. \cite{agarwal2022regret} proposed algorithms based on the posterior sampling and the optimism principle that achieve $\mathcal{\tilde{O}}(\sqrt{T})$ regret with zero constraint violations in the ergodic setting. However, the above model-based algorithms cannot be extended to large state space. In the tabular model-free setup, the algorithm proposed by \cite{wei2022provably} achieves a regret of $\tilde{\mathcal{O}}({T}^{5/6})$ with zero constraint violations. Finally, in the linear CMDP setting, \cite{ghosh2022achieving} achieves $\tilde{\mathcal{O}}(\sqrt{T})$ regret bound with zero constraint violation. Note that the linear CMDP setting assumes that the transition probability has a certain linear structure with a known feature map which is not realistic. Table \ref{table2} summarizes all relevant works. Unfortunately, none of these papers study the infinite horizon average reward CMDPs with general parametrization which is the main focus of our article. 

Additionally, for the weakly communicating setting, \cite{chen2022learning} proposed a model-based algorithm achieving $\mathcal{\tilde{O}}(T^{2/3})$ for both regret and violation in tabular case. \cite{ghosh2022achieving} further extends such result to linear MDP setting with $\mathcal{\tilde{O}}(T^{3/4})$ regret and violation. In general, it is difficult to propose a model-free algorithm with provable guarantees for Constrained MDPs (CMDPs) without considering the ergodic model. \cite{chen2022learning} pointed out several extra challenges in Weakly communicating MDP compared to the ergodic case. For example, there is no uniform bound for the span of the value function for all stationary policies. It is also unclear how to estimate a policy’s bias function accurately without the estimated model, which is an important step for estimating the policy gradient.

\vspace{-.1in}
\section{Formulation}\label{sec:formulation}
\vspace{-.1in}

This paper analyzes an infinite-horizon average reward constrained Markov Decision Process (CMDP) denoted as $\mathcal{M}=(\mathcal{S},\mathcal{A}, r, c, P,\rho)$ where  $\mathcal{S}$ denotes the state space, $\mathcal{A}$ is the action space of size $A$, $r:\mathcal{S}\times\mathcal{A}\rightarrow [0,1]$ is the reward function, $c:\mathcal{S}\times\mathcal{A}\rightarrow [-1,1]$ is the constraint cost function, 
$P:\mathcal{S}\times\mathcal{A}\rightarrow \Delta^{|\mathcal{S}|}$ is the state transition function where $\Delta^{|\mathcal{S}|}$ denotes a probability simplex with dimension $|\mathcal{S}|$, and $\rho\in\Delta^{|\mathcal{S}|}$ is the initial distribution of states. A policy $\pi\in\Pi:\mathcal{S}\rightarrow \Delta^{A}$ maps the current state to an action distribution. The average reward and cost of a policy, $\pi$, is,
\begin{equation}
    \label{Eq_reward}
    J_{g,\rho}^{\pi}\triangleq \lim\limits_{T\rightarrow \infty}\frac{1}{T}\mathbf{E}\bigg[\sum_{t=0}^{T-1}g(s_t,a_t)\bigg|s_0\sim \rho, \pi\bigg]
\end{equation}
where $g=r, c$ for average reward and cost respectively. The expectation is calculated over the distribution of all sampled trajectories $\{(s_t, a_t)\}_{t=0}^{\infty}$ where $a_t\sim \pi(s_t)$, $s_{t+1}\sim P(\cdot|s_t, a_t)$, $\forall t\in\{0, 1, \cdots\}$. For notational convenience, we shall drop the dependence on $\rho$ whenever there is no confusion. Our goal is to maximize the average reward function while ensuring that the average cost is above a given threshold. Without loss of generality, we can mathematically write this problem as, 
\begin{align} 
    \label{eq:def_unparametrized_optimization}
    \begin{split}
        \max_{\pi\in\Pi} ~& J_r^{\pi} ~~
        \text{s.t.} ~ J_c^{\pi}\geq 0
    \end{split}
\end{align}
However, the above problem becomes difficult to handle when the underlying state space, $\mathcal{S}$ is large. Therefore, we consider a class of parametrized  policies, $\{\pi_\theta|\theta\in \Theta\}$ whose elements are indexed by a $\mathrm{d}$-dimensional parameter, $\theta\in\mathbb{R}^{\mathrm{d}}$ where $\mathrm{d}\ll|\mathcal{S}||\mathcal{A}|$. Thus, the original problem in Eq \eqref{eq:def_unparametrized_optimization} can be reformulated as the following parameterized problem.
\begin{align} 
    \label{eq:def_constrained_optimization}
    \begin{split}
        \max_{\theta\in\Theta} ~& J_r^{\pi_{\theta}} 
        ~~\text{s.t.} ~J_c^{\pi_{\theta}}\geq 0
    \end{split}
\end{align}
We denote $J_g^{\pi_\theta} = J_g(\theta)$, $g\in\{r, c\}$ for notational convenience. Let, $P^{\pi_{\theta}}:\mathcal{S}\rightarrow \Delta^{|\mathcal{S}|}$ be a transition function induced by $\pi_\theta$ and defined as, $P^{\pi_{\theta}}(s, s') = \sum_{a\in\mathcal{A}}P(s'|s,a)\pi_{\theta}(a|s)$, $\forall s, s'$. If $\mathcal{M}$ is such that for every policy $\pi$, the function, $P^{\pi}$ is irreducible and aperiodic, then $\mathcal{M}$ is called ergodic.
\begin{assumption}
    \label{ass_1}
    The CMDP $\mathcal{M}$ is ergodic.
\end{assumption}
Ergodicity is a common assumption in the literature \cite{pesquerel2022imed, gong2020duality}. If $\mathcal{M}$ is ergodic, then $\forall \theta$, there exists a unique stationary distribution, $d^{\pi_{\theta}}\in \Delta^{|\mathcal{S}|}$ given as follows.
\begin{align}
    d^{\pi_{\theta}}(s) = \lim_{T\rightarrow \infty}\dfrac{1}{T}\sum_{t=0}^{T-1} \mathrm{Pr}(s_t=s|s_0\sim \rho, \pi_{\theta})
\end{align}
Ergodicity implies that $d^{\pi_{\theta}}$ is independent of the initial distribution, $\rho$, and obeys $P^{\pi_{\theta}}d^{\pi_{\theta}}=d^{\pi_{\theta}}$. Hence, the average reward and cost functions can be expressed as,
\begin{align}
    \label{eq_r_pi_theta}
    J_g(\theta) = \mathbf{E}_{s\sim d^{\pi_{\theta}}, a\sim \pi_{\theta}(s)}[g(s, a)] = (d^{\pi_{\theta}})^T g^{\pi_{\theta}}
\end{align}
where $g^{\pi_{\theta}}(s) \triangleq \sum_{a\in\mathcal{A}}g(s, a)\pi_{\theta}(a|s), ~g\in\{r, c\}$. Note that the functions $J_g(\theta)$, $g\in\{r, c\}$ are also independent of the initial distribution, $\rho$. Furthermore, $\forall \theta$, there exist a function $Q_g^{\pi_{\theta}}: \mathcal{S}\times \mathcal{A}\rightarrow \mathbb{R}$ such that the following Bellman equation is satisfied $\forall (s, a)\in\mathcal{S}\times\mathcal{A}$.
\begin{equation}
    \label{eq_bellman}
    Q_g^{\pi_{\theta}}(s,a)=g(s,a)-J_g(\theta)+\mathbf{E}_{s'\sim P(\cdot|s, a)}\left[V_g^{\pi_{\theta}}(s')\right]
\end{equation}
where $g\in\{r, c\}$ and  $V_g^{\pi_{\theta}}:\mathcal{S}\rightarrow \mathbb{R}$ is given as
$    V_g^{\pi_{\theta}}(s) = \sum_{a\in\mathcal{A}}\pi_{\theta} (a|s)Q_g^{\pi_{\theta}}(s, a), ~\forall s\in\mathcal{S}$. 
Note that if $Q_g^{\pi_{\theta}}$ satisfies $(\ref{eq_bellman})$, then it is also satisfied by $Q_g^{\pi_{\theta}}+c$ for any arbitrary, $c$. To uniquely define the value functions, we assume that $\sum_{s\in\mathcal{S}}d^{\pi_{\theta}}(s)V_g^{\pi_{\theta}}(s)=0$. In this case, $V_g^{\pi_{\theta}}(s)$ is given by,
\begin{equation}
    \label{def_v_pi_theta_s}
    \begin{aligned}
        V_g^{\pi_{\theta}}(s) &= \sum_{t=0}^{\infty} \sum_{s'\in\mathcal{S}}\left[(P^{\pi_{\theta}})^t(s, s') - d^{\pi_{\theta}}(s')\right]g^{\pi_{\theta}}(s')= \sum_{t=0}^{\infty}\mathbf{E}\left[\left\{g(s_t,a_t)-J_g(\theta)\right\}\big\vert s_0=s\right]
    \end{aligned}
\end{equation}
where the expectation is computed over all $\pi_{\theta}$-induced trajectories. In a similar way, $\forall (s, a)$, one can uniquely define $Q_g^{\pi_{\theta}}(s, a)$, $g\in\{r, c\}$ as follows.
\begin{equation}
    \label{def_q_pi_theta_s}
    Q_g^{\pi_{\theta}}(s,a)=\sum_{t=0}^{\infty}\mathbf{E}\left[\big\{g(s_t,a_t)-J_g(\theta)\big\}\big\vert s_0=s,a_0=a\right]
\end{equation}
Moreover, the advantage function $A_g^{\pi_{\theta}}:\mathcal{S}\times \mathcal{A}\rightarrow \mathbb{R}$ is defined such that $A_g^{\pi_{\theta}}(s, a) \triangleq Q_g^{\pi_{\theta}}(s, a) - V_g^{\pi_{\theta}}(s)$, $\forall (s, a)$, $\forall g\in\{r, c\}$.
Assumption \ref{ass_1} also implies the existence of a finite mixing time. Specifically, for an ergodic MDP, $\mathcal{M}$, the mixing time is defined as follows.
\begin{definition}
    The mixing time, $t_{\mathrm{mix}}^\theta$, of the CMDP $\mathcal{M}$ for a parameterized policy, $\pi_\theta$, is defined as, $t_{\mathrm{mix}}^{\theta}\triangleq \min\left\lbrace t\geq 1\big| \norm{(P^{\pi_{\theta}})^t(s, \cdot) - d^{\pi_\theta}}\leq \frac{1}{4}, \forall s\right\rbrace $.
The overall mixing time is $t_{\mathrm{mix}}\triangleq \sup_{\theta\in\Theta} t^{\theta}_{\mathrm{mix}}$. In this paper, $t_{\mathrm{mix}}$ is finite due to ergodicity.
\end{definition}

Mixing time characterizes how fast a CMDP converges to its stationary state distribution, $d^{\pi_\theta}$, under a given policy, $\pi_\theta$. We also define the hitting time as follows.
\begin{definition}
    The hitting time of an ergodic CMDP $\mathcal{M}$ with respect to a policy, $\pi_\theta$, is defined as $t_{\mathrm{hit}}^{\theta}\triangleq  \max_{s\in\mathcal{S}} [d^{\pi_{\theta}}(s)]^{-1}$. 
    The overall hitting time is defined as $t_{\mathrm{hit}}\triangleq \sup_{\theta\in\Theta} t^{\theta}_{\mathrm{hit}} $. In this paper, $t_{\mathrm{hit}}$ is finite due to ergodicity as well.
\end{definition}

Define $\pi^*$ as the optimal solution to the unparameterized problem $\eqref{eq:def_unparametrized_optimization}$. For a given CMDP $\mathcal{M}$, and a time horizon $T$, the regret and constraint violation of any algorithm $\mathbb{A}$ is defined as follows.
\begin{align}
    \mathrm{Reg}_T(\mathbb{A}, \mathcal{M}) \triangleq \sum_{t=0}^{T-1} \left(J_r^{\pi^*}-r(s_t, a_t)\right), \ 
    \mathrm{Vio}_T(\mathbb{A}, \mathcal{M}) \triangleq -\sum_{t=0}^{T-1} c(s_t, a_t)
\end{align}
where the algorithm, $\mathbb{A}$, executes the actions, $\{a_t\}$, $t\in\{0, 1, \cdots \}$ based on the trajectory observed up to time, $t$, and the state, $s_{t+1}$ is decided according to the state transition function, $P$. For simplicity, we shall denote the regret and constraint violation as $\mathrm{Reg}_T$ and $\mathrm{Vio}_T$ respectively. Our goal is to design an algorithm $\mathbb{A}$ that achieves low regret and constraint violation bounds. 

\section{Proposed Algorithm}\label{sec_method}
We solve $\eqref{eq:def_constrained_optimization}$ via a primal-dual algorithm based on the following problem.
\begin{align}
\label{eq:def_saddle_point_opt}
    \max_{\theta\in\Theta}\min_{\lambda\geq 0}~ J_{\mathrm{L}}(\theta, \lambda), 
\end{align}
where $J_{\mathrm{L}}(\theta, \lambda) \triangleq J_r(\theta) +\lambda J_c(\theta)$. The function, $J_{\mathrm{L}}(\cdot, \cdot)$, is called the Lagrange function and $\lambda$ the Lagrange multiplier. Our algorithm updates the pair $(\theta, \lambda)$ following the policy gradient iteration as shown below $\forall k\in\{1,  \cdots, K\}$ with an initial point $(\theta_1, \lambda_1)$, $\lambda_1=0$. 
\begin{align}\label{eq:update}
    \begin{split}
        \theta_{k+1} &= \theta_k+\alpha\nabla_\theta J_{\mathrm{L}}(\theta_k, \lambda_k),~
        \lambda_{k+1} = \mathcal{P}_{[0,\frac{2}{\delta}]}[\lambda_k -\beta J_c(\theta_k)]
	\end{split}
\end{align}
where $\alpha$ and $\beta$ are learning parameters and $\delta$ is the Slater parameter introduced in the following assumption. Finally, for any set, $\Lambda$, $\mathcal{P}_{\Lambda}[\cdot]$ denotes projection onto $\Lambda$. The assumption stated below ensures that we have at least one feasible interior point solution to \eqref{eq:def_unparametrized_optimization}.
\begin{assumption}[Slater condition]
    \label{ass_slater}
    There exists a $\delta\in (0, 1)$ and $\bar{\theta} \in \Theta$ such that $J_{c}(\bar{\theta}) \geq \delta$.
\end{assumption}
Note that in \eqref{eq:update}, the dual update is projected onto the set $[0,\frac{2}{\delta}]$ because the optimal dual variable for the parameterized problem is bounded in Lemma \ref{lem.boundness}. The gradient of $J_{\mathrm{L}}(\cdot, \lambda)$ can be computed by invoking a variant of the well-known policy gradient theorem \cite{sutton1999policy}.

\begin{lemma}
    \label{lemma_grad_compute}
    The gradient of  $J_{\mathrm{L}}(\cdot, \lambda)$ is computed as,
    \begin{align*}
        \nabla_{\theta} J_{\mathrm{L}}(\theta, \lambda)
        =\mathbf{E}_{s\sim d^{\pi_{\theta}},a\sim\pi_{\theta}(s)}\big[A_{\mathrm{L},\lambda}^{\pi_{\theta}}(s,a)\nabla_{\theta}\log\pi_{\theta}(a|s)\big]
     \end{align*}
\end{lemma}

\begin{algorithm}[h]
    \caption{Primal-Dual Parameterized Policy Gradient}
    \label{alg:PG_MAG}
    \begin{algorithmic}[1]
        \STATE \textbf{Input:} Episode length $H$, learning rates $\alpha,\beta$, initial parameters $\theta_1,\lambda_1$, initial state $s_0 \sim \rho(\cdot)$, 
        \vspace{0.1cm}
        \STATE $K=T/H$
	\FOR{$k\in\{1, \cdots, K\}$}
            \STATE $\mathcal{T}_k\gets \phi$
            \FOR{$t\in\{(k-1)H, \cdots, kH-1\}$}
                \STATE Execute $a_t\sim \pi_{\theta_k}(\cdot|s_t)$
                \STATE Observe $r(s_t,a_t)$, $c(s_t,a_t)$ and  $s_{t+1}$
                \STATE $\mathcal{T}_k\gets \mathcal{T}_k\cup \{(s_t, a_t)\}$
            \ENDFOR	
            \FOR{$t\in\{(k-1)H, \cdots, kH-1\}$}
                \STATE Obtain $\hat{A}_{\mathrm{L},\lambda_k}^{\pi_{\theta_k}}(s_t, a_t)$ via Algorithm \ref{alg:estQ} and $\mathcal{T}_k$ 
            \ENDFOR
            \vspace{0.1cm}
            \STATE Compute $\omega_k$ using \eqref{eq_grad_estimate} 
		  \STATE Update the parameters:
		  \begin{align}
                \label{udpates_algorotihm}
                &\theta_{k+1}=\theta_k+\alpha\omega_k, \\
                &\lambda_{k+1}= \mathcal{P}_{[0,\frac{2}{\delta}]}\big[\lambda_k -\beta \hat{J}_c(\theta_k)\big] \nonumber\\
                & \text{where } \hat{J}_c(\theta_k)=\frac{1}{H-N}\sum_{t=(k-1)H+N}^{kH-1}c(s_t, a_t) \nonumber
		  \end{align}
            \vspace{-0.2cm}
        \ENDFOR
    \end{algorithmic}
\end{algorithm}

where $\forall (s, a)$, $A_{\mathrm{L},\lambda}^{\pi_\theta}(s,a)\triangleq A_r^{\pi_\theta}(s, a)+\lambda A_c^{\pi_\theta}(s, a)$,  and $\{A_g^{\pi_\theta}\}_{g\in\{r, c\}}$ are the advantage functions corresponding to reward and cost. In typical RL scenarios, learners do not have access to the state transition function, $P$, and thereby to the functions $d^{\pi_\theta}$ and $A^{\pi_\theta}_{\mathrm{L}, \lambda}$. This makes computing the exact gradient a difficult task. In Algorithm \ref{alg:PG_MAG}, we demonstrate how one can still obtain good estimates of the gradient using sampled trajectories.

Algorithm \ref{alg:PG_MAG} runs $K$ epochs, each of duration $H=16t_{\mathrm{hit}}t_{\mathrm{mix}}T^{\xi}(\log T)^2$ where $\xi\in (0, 1)$ defines a constant whose value is specified later. Clearly, $K=T/H$. Note that the learner is assumed to know the horizon length, $T$. This can be relaxed utilizing the well-known doubling trick \cite{lattimore2020bandit}. Additionally, it is assumed that the algorithm is aware of the mixing time and the hitting time. This assumption is common in the literature \cite{bai2023regret, wei2020model}. The first step in obtaining a gradient estimate is estimating the advantage value for a given pair $(s, a)$. This can be accomplished via Algorithm \ref{alg:estQ}. At the $k$th epoch, a $\pi_{\theta_k}$-induced trajectory, $\mathcal{T}_k=\{(s_t, a_t)\}_{t=(k-1)H}^{kH-1}$ is obtained and subsequently passed to Algorithm \ref{alg:estQ} that searches for subtrajectories within it that start with a given state $s$, are of length $N=4t_{\mathrm{mix}}(\log T)$, and are at least $N$ distance apart from each other. Assume that there are $M$ such subtrajectories. Let the total reward and cost of the $i$th subtrajectory be $\{r_i, c_i\}$ respectively and $\tau_i$ be its starting time. The value function estimates for the $k$th epoch are 
\begin{align}
\label{eq:eq_13_new_rev}
        \hat{Q}_g^{\pi_{\theta_k}}(s, a) = \dfrac{1}{\pi_{\theta_k}(a|s)} \left[\dfrac{1}{M}\sum_{i=1}^M g_i \mathrm{1}(a_{\tau_i}=a)\right],
    ~\hat{V}_g^{\pi_{\theta_k}}(s) =  \dfrac{1}{M}\sum_{i=1}^M g_i , ~~\forall g\in\{r, c\}
\end{align}

This leads to the following advantage estimator.
\begin{align}
    \hat{A}^{\pi_{\theta_k}}_{\mathrm{L}, \lambda_k} (s, a) = \hat{A}^{\pi_{\theta_k}}_{r} (s, a) + \lambda_k \hat{A}^{\pi_{\theta_k}}_{c} (s, a), 
\end{align}

where $\hat{A}^{\pi_{\theta_k}}_{g}(s, a) = \hat{Q}^{\pi_{\theta_k}}_{g}(s, a) - \hat{V}^{\pi_{\theta_k}}_{g}(s)$,  $g\in\{r,c\}$.  Finally, the gradient estimator is,
\begin{align}\label{eq_grad_estimate}
   & \omega_k\triangleq\hat{\nabla}_{\theta} J_{\mathrm{L}}(\theta_k,\lambda_k)
    =\dfrac{1}{H}\sum_{t=t_k}^{t_{k+1}-1}\hat{A}_{\mathrm{L},\lambda_k}^{\pi_{\theta_k}}(s_{t}, a_{t})\nabla_{\theta}\log \pi_{\theta_k}(a_{t}|s_{t})
\end{align}
where $t_k=(k-1)H$ is the starting time of the $k$th epoch. The parameters are updated following \eqref{udpates_algorotihm}. To update the Lagrange multiplier, we need an estimation of $J_c(\theta_k)$, which is obtained as the average cost of the $k$th epoch. It should be noted that we remove the first $N$ samples from the $k$th epoch because we require the state distribution emanating from the remaining samples to be close enough to the stationary distribution $d^{\pi_{\theta_k}}$, which is the key to make $\hat{J}_c(\theta_k)$ close to $J_c(\theta_k)$. The following lemma demonstrates that $\hat{A}_{\mathrm{L},\lambda_k}^{\pi_{\theta_k}}(s, a)$ is a good estimator of $A_{\mathrm{L},\lambda_k}^{\pi_{\theta_k}}(s, a)$.

\begin{algorithm}[ht]
    \caption{Advantage Estimation}
    \label{alg:estQ}
    \begin{algorithmic}[1]
        \STATE \textbf{Input:} Trajectory $(s_{t_1}, a_{t_1},\ldots, s_{t_2}, a_{t_2})$, state $s$, action $a$, Lagrange multiplier $\lambda$, and parameter $\theta$
        \STATE \textbf{Initialize:} $M \leftarrow 0$, $\tau\leftarrow t_1$
	\STATE \textbf{Define:} $N=4t_{\mathrm{mix}}\log_2T$.
	\vspace{0.1cm}
 
	\WHILE{$\tau\leq t_2-N$}
		\IF{$s_{\tau}=s$}
			\STATE $M\leftarrow M+1$, ~$\tau_M\gets \tau$ 
                \STATE $g_M\gets \sum_{t=\tau}^{\tau+N-1}g(s_t, a_t)$, ~$\forall g\in\{r, c\}$
			\STATE $\tau\leftarrow\tau+2N$.
		\ELSE
                \STATE {$\tau\leftarrow\tau+1$.}
            \ENDIF
	\ENDWHILE
 
        \vspace{0.1cm}
        \IF{$M>0$}
            \STATE Compute $\hat{Q}_g(s, a)$, $\hat{V}_g(s)$ via \eqref{eq:eq_13_new_rev}, $\forall g\in\{r, c\}$
        \ELSE
            \STATE $\hat{V}_g(s)=0$, $\hat{Q}_g(s, a)=0$, ~~$\forall g\in\{r, c\}$
        \ENDIF
	\STATE \textbf{return}  $(\hat{Q}_r(s,a)-\hat{V}_r(s))+\lambda(\hat{Q}_c(s,a)-\hat{V}_c(s))$ 
    \end{algorithmic}
\end{algorithm}

\begin{lemma}
    \label{lemma_good_estimator}
    The following inequality holds $\forall k$, $\forall (s, a)$ and sufficiently large $T$.
    \begin{equation}\label{eq_good_estimator}
        \mathbf{E}\bigg[\bigg(\hat{A}_{\mathrm{L}, \lambda_k}^{\pi_{\theta_k}}(s, a)-A^{\pi_{\theta_k}}_{\mathrm{L}, \lambda_k}(s, a)\bigg)^2\bigg] \leq \mathcal{O}\left(\dfrac{ t_{\mathrm{hit}}N^3\log T}{\delta^2 H\pi_{\theta_k}(a|s)}\right) 
        =\mathcal{O}\left(\dfrac{t_{\mathrm{mix}}^2(\log T)^2}{\delta^2T^{\xi}\pi_{\theta_k}(a|s)}\right) 
    \end{equation}
\end{lemma}

Lemma \ref{lemma_good_estimator} shows that the $L_2$ error of our proposed advantage estimator can be bounded above as $\Tilde{\mathcal{O}}(T^{-\xi})$. We later utilize the above result to prove the goodness of the gradient estimator. It is to be clarified that our Algorithm \ref{alg:estQ} is inspired by Algorithm 2 of \cite{wei2020model}. However, while the authors of \cite{wei2020model} choose $H=\tilde{\mathcal{O}}(1)$, we adapt $H=\tilde{\mathcal{O}}(T^{\xi})$. This subtle change is important in proving a sublinear regret for general parametrization.
 
\vspace{-.1in}
\section{Global Convergence Analysis}
\label{sec_convergence}
\vspace{-.1in}

This section first shows that the sequence $\{\theta_k, \lambda_k\}_{k=1}^K$ produced by Algorithm \ref{alg:PG_MAG} is such that their associated Lagrange sequence $\{J_{\mathrm{L}}(\theta_k, \lambda_k)\}_{k=1}^{\infty}$ converges globally. By expanding the Lagrange function, we then exhibit convergence of each of its components $\{J_g(\theta_k, \lambda_k)\}_{k=1}^K$, $g\in\{r,c\}$. This is later used for regret and constraint violation analysis. Before delving into the details, we would like to state a few necessary assumptions.

\begin{assumption}
    \label{ass_score}
    The score function (stated below) is $G$-Lipschitz and $B$-smooth. Specifically, $\forall \theta, \theta_1,\theta_2 \in\Theta$, and $\forall (s,a)$, the following inequalities hold.
    \begin{align*}
            \Vert \nabla_\theta\log\pi_\theta(a\vert s)\Vert\leq G,\ 
            \Vert \nabla_\theta\log\pi_{\theta_1}(a\vert s)-\nabla_\theta\log\pi_{\theta_2}(a\vert s)\Vert\leq B\Vert \theta_1-\theta_2\Vert\quad
    \end{align*}
\end{assumption}

\begin{remark}
    The Lipschitz and smoothness properties of the score function are commonly assumed for policy gradient analyses \cite{Alekh2020, Mengdi2021, liu2020improved}. These assumptions hold for simple parameterization classes such as Gaussian policies.  
\end{remark}

Note that by combining Assumption \ref{ass_score} with Lemma \ref{lemma_good_estimator} and using the gradient estimator as given in $(\ref{eq_grad_estimate})$, one can deduce the following result.

\begin{lemma}
    \label{lemma_grad_est_bias}
    The following inequality holds $\forall k$ provided that assumptions \ref{ass_1} and \ref{ass_score} are true.
    \begin{align}
        \mathbf{E}\left[\norm{\omega_k-\nabla_{\theta}J_{\mathrm{L}}(\theta_k,\lambda_k)}^2\right]\leq \tilde{\mathcal{O}}\left(\delta^{-2}AG^2t_{\mathrm{mix}}^2T^{-\xi}\right)
    \end{align}
\end{lemma}
Lemma \ref{lemma_grad_est_bias} claims that the gradient estimation error can be bounded as $\tilde{\mathcal{O}}(T^{-\xi})$. We will use this result later to prove the global convergence of our algorithm.

\begin{assumption}
    \label{ass_transfer_error}
    Let the transferred compatible function approximation error be defined as follows.
    \begin{align}
        \label{eq:transfer_error}
	\begin{split}
            L_{d^{\pi^*},\pi^*}(\omega^*_{\theta, \lambda},\theta, \lambda) = \mathbf{E}_{s\sim d^{\pi^*}}\mathbf{E}_{a\sim\pi^*(s)}\bigg[\bigg(\nabla_\theta\log\pi_{\theta}(a\vert s)\cdot
			\omega^*_{\theta, \lambda}-A_{\mathrm{L}, \lambda}^{\pi_\theta}(s,a)\bigg)^2\bigg]
	\end{split}
    \end{align}
    where $\pi^*$ is the optimal solution of unparameterized problem in $\eqref{eq:def_unparametrized_optimization}$ and
    \begin{align}
        \label{eq:NPG_direction}
	\begin{split}
            \omega^*_{\theta, \lambda}=\arg \min_{\omega\in\mathbb{R}^{\mathrm{d}}}~\mathbf{E}_{s\sim d^{\pi_{\theta}}}\mathbf{E}_{a\sim\pi_{\theta}( s)}\bigg[\bigg(\nabla_\theta\log\pi_{\theta}(a\vert s)\cdot\omega-A_{\mathrm{L}, \lambda}^{\pi_{\theta}}(s,a)\bigg)^2\bigg]
	\end{split}
    \end{align}
    We assume that $L_{d^{\pi^*},\pi^*}(\omega^*_{\theta, \lambda},\theta, \lambda)\leq \epsilon_{\mathrm{bias}}$, $\lambda\in [0, \frac{2}{\delta}]$ and $\theta\in\Theta$ where $\epsilon_{\mathrm{bias}}$ is a positive constant.
\end{assumption}

\begin{remark}
    The transferred compatible function approximation error quantifies the expressivity of the parameterized policy class. We can show that $\epsilon_{\mathrm{bias}}=0$ for softmax parameterization \cite{agarwal2021theory} and linear MDPs \cite{Chi2019}. If the policy class is restricted, i.e., it does not contain all stochastic policies, $\epsilon_{\mathrm{bias}}$ turns out to be strictly positive. However, if the policy class is parameterized by a rich neural network, then $\epsilon_{\mathrm{bias}}$ can be assumed to be negligibly small \cite{Lingxiao2019}. Such assumptions are common \cite{liu2020improved,agarwal2021theory}. 
\end{remark}

\begin{remark}
    Note that $\omega^*_{\theta, \lambda}$ defined in \eqref{eq:NPG_direction} can be written as,
    \begin{align*}
        \omega^*_{\theta, \lambda} = F_{\rho}(\theta)^{\dagger} \mathbf{E}_{s\sim d_\rho^{\pi_{\theta}}}\mathbf{E}_{a\sim\pi_{\theta}( s)}\left[\nabla_{\theta}\log\pi_{\theta}(a|s)A_{\mathrm{L},\lambda}^{\pi_{\theta}}(s, a)\right]
    \end{align*}
    where $\dagger$ is the Moore-Penrose pseudoinverse and $F_{\rho}(\theta)$ is the Fisher information matrix  defined as,
    \begin{align*}
        F_{\rho}(\theta) = \mathbf{E}_{s\sim d_\rho^{\pi_{\theta}}}\mathbf{E}_{a\sim\pi_{\theta}(\cdot\vert s)}[\nabla_{\theta}\log\pi_{\theta}(a|s)(\nabla_{\theta}\log\pi_{\theta}(a|s))^T]
    \end{align*}
\end{remark}

\begin{assumption}
    \label{ass_4}
    There exists a constant $\mu_F>0$ such that $F_{\rho}(\theta)-\mu_F I_{\mathrm{d}}$ is positive semidefinite where $I_{\mathrm{d}}$ is an identity matrix of dimension, $\mathrm{d}$.
\end{assumption}

Assumption \ref{ass_4} is also called Fisher-non-degenerate policy assumption and is quite common in the literature \cite{liu2020improved, yuan2022general, fatkhullin2023stochastic} in the policy gradient analysis. \cite{liu2020improved}[Assumption 2.1] provided a detailed discussion on the requirement of policy class to satisfy the assumption \ref{ass_4}. Moreover, \cite{mondal2023mean} describes a class of policies that obeys assumptions $\ref{ass_score}-\ref{ass_4}$ simultaneously. The Lagrange difference lemma stated below is important in establishing global convergence.

\begin{lemma}
    \label{lem_performance_diff}
    With a slight abuse of notation, let $J_{\mathrm{L}}(\pi,\lambda)=J_r^{\pi}+\lambda J_c^{\pi}$. For any two policies $\pi$, $\pi'$, the following result holds $\forall \lambda>0$.
    \begin{equation*}
        J_{\mathrm{L}}(\pi,\lambda)-J_{\mathrm{L}}(\pi',\lambda)=\mathbf{E}_{s\sim d^{\pi}}\mathbf{E}_{a\sim\pi( s)}\big[A_{\mathrm{L},\lambda}^{\pi'}(s,a)\big]
    \end{equation*}
\end{lemma}

We now present a general framework for the convergence analysis of Algorithm \ref{alg:PG_MAG}. 

\begin{lemma}
    \label{lem_framework} 
    If the policy parameters, $\{\theta_k, \lambda_k\}_{k=1}^K$ are updated via $\eqref{udpates_algorotihm}$ and assumptions \ref{ass_score}, \ref{ass_transfer_error},and \ref{ass_4} hold, then we have the following inequality for any $K$,
    \begin{equation}\label{eq:general_bound}
        \begin{aligned}
            \frac{1}{K}\sum_{k=1}^{K}\mathbf{E}\bigg(J_{\mathrm{L}}(\pi^*, \lambda_k)-J_{\mathrm{L}}(\theta_k,\lambda_k)\bigg)&\leq \sqrt{\epsilon_{\mathrm{bias}}}+\frac{G}{K}\sum_{k}^{K}\mathbf{E}\Vert(\omega_k-\omega^*_k)\Vert+\frac{B\alpha}{2K}\sum_{k=1}^{K}\mathbf{E}\Vert\omega_k\Vert^2\nonumber
            \\
            &+\frac{1}{\alpha K}\mathbf{E}_{s\sim d^{\pi^*}}[KL(\pi^*(\cdot\vert s)\Vert\pi_{\theta_1}(\cdot\vert s))]		
        \end{aligned}            
    \end{equation}
    where $\omega^*_k:=\omega^*_{\theta_k, \lambda_k}$, $\omega^*_{\theta_k,\lambda_k}$ is defined in \eqref{eq:NPG_direction}, and $\pi^*$ is the optimal solution to the  problem $\eqref{eq:def_unparametrized_optimization}$.
\end{lemma}

Lemma \ref{lem_framework} proves that the optimality error of the Lagrange sequence can be bounded by the average first-order and second-order norms of the intermediate gradients. Note the presence of $\epsilon_{\mathrm{bias}}$ in the result. If the policy class is severely restricted, the optimality bound loses its importance. Consider the expectation of the second term in \eqref{eq:general_bound}. Note that,
\begin{equation*}\label{eq_second_term_bound}
    \begin{aligned}
    &\bigg(\frac{1}{K}\sum_{k=1}^{K}\mathbf{E}\Vert\omega_k-\omega^*_k\Vert\bigg)^2\leq \frac{1}{K}\sum_{k=1}^{K}\mathbf{E}\bigg[\Vert\omega_k-\omega^*_k\Vert^2\bigg]\nonumber=\frac{1}{K}\sum_{k=1}^{K}\mathbf{E}\bigg[\Vert\omega_k-F_\rho(\theta_k)^\dagger\nabla_\theta J_{\mathrm{L}}(\theta_k,\lambda_k)\Vert^2\bigg]\\
    &\leq \frac{2}{K}\sum_{k=1}^{K}\Bigg\lbrace \mathbf{E}\bigg[\Vert\omega_k-\nabla_\theta J_{\mathrm{L}}(\theta_k,\lambda_k)\Vert^2\bigg]\nonumber+ \mathbf{E}\bigg[\Vert \nabla_\theta J_{\mathrm{L}}(\theta_k,\lambda_k)- F_\rho(\theta_k)^\dagger\nabla_\theta J_{\mathrm{L}}(\theta_k,\lambda_k)\Vert^2\bigg]\Bigg\rbrace \\
    &\overset{(a)}{\leq} \frac{2}{K}\sum_{k=1}^{K}\mathbf{E}\bigg[\Vert \omega_k-\nabla_\theta J_{\mathrm{L}}(\theta_k,\lambda_k)\Vert^2\bigg]\nonumber+\frac{2}{K}\sum_{k=1}^{K}\left(1+\dfrac{1}{\mu_F^2}\right)\mathbf{E}\bigg[\Vert \nabla_\theta J_{\mathrm{L}}(\theta_k,\lambda_k)\Vert^2\bigg] 
    \end{aligned}
\end{equation*}
    
where $(a)$ follows from Assumption \ref{ass_4}. The expectation of the third term in \eqref{eq:general_bound} can be bounded as
\begin{equation}\label{eq_third_term_bound}
    \dfrac{1}{K}\sum_{k=1}^{K}\mathbf{E}\bigg[\Vert\omega_k\Vert^2\bigg]\leq \dfrac{1}{K}\sum_{k=1}^{K}\mathbf{E}\left[\norm{\nabla_\theta J_{\mathrm{L}}(\theta_k,\lambda_k)}^2\right] 
        \nonumber+ \dfrac{1}{K}\sum_{k=1}^{K}\mathbf{E}\bigg[\Vert \omega_k-\nabla_\theta J_{\mathrm{L}}(\theta_k,\lambda_k)\Vert^2\bigg] 
\end{equation}
In both $\eqref{eq_second_term_bound}$ and $\eqref{eq_third_term_bound}$,  $\mathbf{E}\norm{\omega_k-\nabla_{\theta}J_{\mathrm{L}}(\theta_k,\lambda_k)}^2$ is bounded above by Lemma \ref{lemma_grad_est_bias}. To bound the term, $\mathbf{E}\norm{\nabla_\theta J_{\mathrm{L}}(\theta_k,\lambda_k)}^2$, the following lemma is applied.
 
\begin{lemma} 
    \label{lemma:41ss}
    Let $J_{g}(\cdot)$ be $L$-smooth, $\forall g\in\{r, c\}$ and $\alpha = \frac{1}{4L(1+\frac{2}{\delta})}$. Then the following holds.
    \begin{equation}
        \label{eq_24_new}
        \dfrac{1}{K}\sum_{k=1}^{K} \norm{\nabla_\theta J_{\mathrm{L}}(\theta_k,\lambda_k)}^2 \leq \frac{288L}{\delta^2 K}
            +\dfrac{3}{K}\sum_{k=1}^{K}\Vert\nabla_\theta J_{\mathrm{L}}(\theta_k,\lambda_k) - \omega_k\Vert^2+\beta 
    \end{equation}
\end{lemma}
Note the presence of $\beta$ in \eqref{eq_24_new}. To ensure convergence, $\beta$ must be a function of $T$. 
Invoking Lemma \ref{lemma_grad_est_bias}, we get the following relation under the same set of assumptions and the choice of parameters as in Lemma \ref{lemma:41ss}. 
\begin{align}
    \label{eq_33}
    \begin{split}
        \dfrac{1}{K}\sum_{k=1}^{K} \mathbf{E}\norm{\nabla_\theta J_{\mathrm{L}}(\theta_k,\lambda_k)}^2\leq \tilde{\mathcal{O}}\left(\dfrac{AG^2t_{\mathrm{mix}}^2}{\delta^2 T^{\xi}}\right)+\tilde{\mathcal{O}}\left(\dfrac{Lt_{\mathrm{mix}}t_{\mathrm{hit}}}{\delta^2 T^{1-\xi}}\right)+\beta
    \end{split}
\end{align}
Applying Lemma \ref{lemma_grad_est_bias} and \eqref{eq_33} in \eqref{eq_third_term_bound}, we arrive at,
\begin{align}
    \label{eq_34}
    \dfrac{1}{K}\sum_{k=1}^{K}\mathbf{E}\bigg[\Vert\omega_k\Vert^2\bigg]&\leq \tilde{\mathcal{O}}\left(\dfrac{AG^2t_{\mathrm{mix}}^2}{\delta^2 T^{\xi}}+\dfrac{Lt_{\mathrm{mix}}t_{\mathrm{hit}}}{\delta^2 T^{1-\xi}}\right)+\beta
\end{align}
Similarly, using $(\ref{eq_second_term_bound})$, we deduce the following.
\begin{align}
    \label{eq_35}
    \begin{split}
        \frac{1}{K}\sum_{k=1}^{K}\mathbf{E}\Vert\omega_k-\omega^*_k\Vert \leq \left(1+\dfrac{1}{\mu_F}\right)\sqrt{\beta}+ \left(1+\dfrac{1}{\mu_F}\right)\tilde{\mathcal{O}}\left(\dfrac{\sqrt{A}Gt_{\mathrm{mix}}}{\delta T^{\xi/2}}+\dfrac{\sqrt{Lt_{\mathrm{mix}}t_{\mathrm{hit}}}}{\delta T^{(1-\xi)/2}}\right)
    \end{split}
\end{align}
Inequalities \eqref{eq_34} and \eqref{eq_35} lead to the following global convergence of the Lagrange function.

\begin{lemma}
    \label{Lemma_global_Lagrange}
    Let $\{\theta_k\}_{k=1}^{K}$ be as described in Lemma \ref{lem_framework}. If assumptions \ref{ass_1}$-$\ref{ass_4} hold,  $\{J_{g}(\cdot)\}_{g\in\{r,c\}}$ are $L$-smooth functions, $\alpha=\frac{1}{4L(1+\frac{2}{\delta})}$,  $K=\frac{T}{H}$, and $H=16t_{\mathrm{mix}}t_{\mathrm{hit}}T^{\xi}(\log_2 T)^2$, then 
    \begin{equation*}
        \begin{split}
            &\frac{1}{K}\sum_{k=1}^{K} \mathbf{E}\bigg(J_{\mathrm{L}}(\pi^*,\lambda_k)-J_{\mathrm{L}}(\theta_k,\lambda_k)\bigg)\leq  G\left(1+\dfrac{1}{\mu_F}\right)\tilde{\mathcal{O}}\left(\sqrt{\beta}+\dfrac{\sqrt{A}Gt_{\mathrm{mix}}}{\delta T^{\xi/2}}+\dfrac{\sqrt{Lt_{\mathrm{mix}}t_{\mathrm{hit}}}}{\delta T^{(1-\xi)/2}}\right)\\
            &+\dfrac{B}{L}\tilde{\mathcal{O}}\left(\dfrac{AG^2t_{\mathrm{mix}}^2}{\delta^2 T^{\xi}}+\dfrac{Lt_{\mathrm{mix}}t_{\mathrm{hit}}}{\delta^2 T^{1-\xi}}+\beta\right)+\mathcal{\tilde{O}}\bigg(\frac{Lt_{\mathrm{mix}}t_{\mathrm{hit}}\mathbf{E}_{s\sim d^{\pi^*}}[KL(\pi^*(\cdot\vert s)\Vert\pi_{\theta_1}(\cdot\vert s))]}{T^{1-\xi}\delta }\bigg) + \sqrt{\epsilon_{\mathrm{bias}}}
        \end{split}
    \end{equation*}
\end{lemma}
Lemma \ref{Lemma_global_Lagrange} establishes that the average difference between $J_{\mathrm{L}}(\pi^*, \lambda_k)$ and $J_{\mathrm{L}}(\theta_k, \lambda_k)$ is $\tilde{\mathcal{O}}(\sqrt{\beta}+T^{-\xi/2}+T^{-(1-\xi)/2})$. Expanding the function, $J_{\mathrm{L}}$, and utilizing the update rule of the Lagrange multiplier, we achieve the global convergence for the objective and the constraint in Theorem \ref{thm_global_convergence} (stated below). In its proof, Lemma \ref{lem.constraint} (stated in the appendix) serves as an important tool in disentangling the convergence rates of regret and constraint violation. Interestingly, Lemma 18 is built upon the strong duality property of the unparameterized optimization \eqref{eq:def_unparametrized_optimization} and has no apparent direct connection with the parameterized setup.

\begin{theorem}
    \label{thm_global_convergence}
    Consider the same parameters as in Lemma \ref{Lemma_global_Lagrange} and set $\beta=T^{-2/5}$, $\xi = 2/5$. We have,
    \begin{align*}
        \begin{split}
            \frac{1}{K}&\sum_{k=1}^{K}\mathbf{E}\bigg(J_r^{\pi^*}-J_r(\theta_k)\bigg)\leq \sqrt{\epsilon_{\mathrm{bias}}}+\dfrac{\sqrt{A}G^2t_{\mathrm{mix}}}{\delta}\left(1+\dfrac{1}{\mu_F}\right)\tilde{\mathcal{O}}\left(T^{-1/5}\right)\\
            \frac{1}{K}&\sum_{k=1}^{K}\mathbf{E}\bigg(-J_c(\theta_k)\bigg)\leq \delta\sqrt{\epsilon_{\mathrm{bias}}} + \tilde{\mathcal{O}}\left(\dfrac{t_{\mathrm{mix}}t_{\mathrm{hit}}}{\delta T^{1/5}}\right) + \sqrt{A}G^2t_{\mathrm{mix}}\left(1+\dfrac{1}{\mu_F}\right)\tilde{\mathcal{O}}\left( T^{-1/5}\right)
        \end{split}
    \end{align*}
  where $\pi^*$ is a solution to \eqref{eq:def_unparametrized_optimization}. In the above bounds, we write only the dominating terms of $T$.
\end{theorem}
Theorem \ref{thm_global_convergence} establishes $\tilde{\mathcal{O}}(T^{-1/5})$ convergence rates for both the objective and the constraint violation.

 \vspace{-.1in}
\section{Regret and Violation Analysis}
\label{sec_regret}
 \vspace{-.1in}
In this section, we use the convergence result of the previous section to bound the expected regret and constraint violation of Algorithm \ref{alg:PG_MAG}. Note that the regret and constraint violation decompose as,
\begin{equation}
    \begin{split}
        &\mathrm{Reg}_T = \sum_{t=0}^{T-1} \left(J_r^{\pi^*} - r(s_t, a_t)\right)=H\sum_{k=1}^{K}\left(J_r^{\pi^*}-J({\theta_k})\right)+\sum_{k=1}^{K}\sum_{t\in\mathcal{I}_k} \left(J(\theta_k)-r(s_t, a_t)\right)\\
        &\mathrm{Vio}_T =\sum_{t=0}^{T-1} \left(-c(s_t, a_t)\right)=H\sum_{k=1}^{K}\left(-J_c({\theta_k})\right)+\sum_{k=1}^{K}\sum_{t\in\mathcal{I}_k} \left(J_c(\theta_k)-c(s_t, a_t)\right)\nonumber
    \end{split}
\end{equation}
    
where $\mathcal{I}_k\triangleq \{(k-1)H, \cdots, kH-1\}$. Observe that the expectation of the first terms in regret and violation can be bounded by Theorem \ref{thm_global_convergence}. The expectation of the second term in regret and violation can be expanded as follows,
\begin{equation}\label{eq_38}
    \begin{aligned}
        &\mathbf{E}\left[\sum_{k=1}^{K}\sum_{t\in\mathcal{I}_k} \left(J_g(\theta_k)-g(s_t, a_t)\right)\right]\overset{(a)}{=}\mathbf{E}\left[\sum_{k=1}^{K}\sum_{t\in\mathcal{I}_k} \mathbf{E}_{s'\sim P(\cdot|s_t, a_t)}[V_g^{\pi_{\theta_k}}(s')]-Q_g^{\pi_{\theta_k}}(s_t, a_t)\right]\\
        &\overset{(b)}{=}\mathbf{E}\left[\sum_{k=}^{K}\sum_{t\in\mathcal{I}_k} V_g^{\pi_{\theta_k}}(s_{t+1})-V_g^{\pi_{\theta_k}}(s_t)\right]=\mathbf{E}\left[\sum_{k=1}^{K} V_g^{\pi_{\theta_k}}(s_{kH})-V_g^{\pi_{\theta_k}}(s_{(k-1)H})\right]\\
        &=\mathbf{E}\left[\sum_{k=1}^{K-1} V_g^{\pi_{\theta_{k+1}}}(s_{kH})-V_g^{\pi_{\theta_k}}(s_{kH})\right]+\mathbf{E}\left[V_g^{\pi_{\theta_K}}(s_{T})-V_g^{\pi_{\theta_0}}(s_{0})\right]
    \end{aligned}
\end{equation}

where $g\in\{r, c\}$. Equality $(a)$ uses the Bellman equation and $(b)$ follows from the definition of $Q_g$. The first term in the last line of Eq. \eqref{eq_38} can be upper bounded by Lemma \ref{lemma_last} (stated below). On the other hand, the second term can be upper bounded as $\mathcal{O}(t_{\mathrm{mix}})$ using Lemma \ref{lemma_aux_2}.

\begin{lemma}
    \label{lemma_last}
    If assumptions \ref{ass_1} and \ref{ass_score} hold, then for $K=\frac{T}{H}$ where $H=16t_{\mathrm{mix}}t_{\mathrm{hit}}T^{\frac{2}{5}}(\log_2 T)^2$, the following inequalities hold $\forall k$, $\forall (s, a)$ and sufficiently large $T$: 
    \begin{align*}
        &(a)\,|\pi_{\theta_{k+1}}(a|s)-\pi_{\theta_{k}}(a|s)|\leq G\norm{\theta_{k+1}-\theta_k}\\
        &(b)\, \sum_{k=1}^{K}\mathbf{E}|J_g(\theta_{k+1})-J_g(\theta_k)|
        \leq\mathcal{\tilde{O}}\left(\dfrac{\alpha AG}{\delta t_{\mathrm{hit}}}\left[\left(\sqrt{A}G t_{\mathrm{mix}}+\delta\right)T^{\frac{2}{5}}+\sqrt{Lt_{\mathrm{mix}}t_{\mathrm{hit}}}T^{\frac{3}{10}}\right]\right)\\
        &(c)\, \sum_{k=1}^K\mathbf{E}|V_g^{\pi_{\theta_{k+1}}}(s_k) - V_g^{\pi_{\theta_{k}}}(s_k)| 
        \leq  \mathcal{\tilde{O}}\left(\dfrac{\alpha AG t_{\mathrm{mix}}}{\delta t_{\mathrm{hit}}}\left[\left(\sqrt{A}G t_{\mathrm{mix}}+\delta\right)T^{\frac{2}{5}}+\sqrt{Lt_{\mathrm{mix}}t_{\mathrm{hit}}}T^{\frac{3}{10}}\right]\right)
    \end{align*}
    where $g\in\{r, c\}$, and $\{s_k\}_{k=1}^K$ is an arbitrary sequence of states.
\end{lemma}

Lemma \ref{lemma_last}  states that the obtained policy parameters are such that the average consecutive difference in the sequence $\{J_g(\theta_k)\}_{k=1}^K$, $g\in\{r, c\}$ decreases with time horizon, $T$. We would like to emphasize that Lemma \ref{lemma_last} works for both reward and constraint functions. Hence, we can prove our regret guarantee and constraint violation as shown below.

\begin{theorem}
    \label{thm_regret}
    If assumptions \ref{ass_1}$-$\ref{ass_4} hold, $J_g(\cdot)$'s are $L$-smooth, $\forall g\in\{r, c\}$ and $T$ are sufficiently large, then our proposed Algorithm \ref{alg:PG_MAG} achieves the following expected regret and constraint violation bounds with learning rates $\alpha=\frac{1}{4L(1+\frac{2}{\delta})}$ and $\beta=T^{-2/5}$.
    \begin{align}\label{eq_regret}
        &\mathbf{E}\left[\mathrm{Reg}_{T}\right] \leq T\sqrt{\epsilon_{\mathrm{bias}}} +\tilde{\mathcal{O}}(T^{4/5})+\mathcal{O}(t_{\mathrm{mix}})\\
        &\mathbf{E}\left[\mathrm{Vio}_{T}\right] \leq T\delta\sqrt{\epsilon_{\mathrm{bias}}} +\tilde{\mathcal{O}}(T^{4/5})+\mathcal{O}(t_{\mathrm{mix}})
    \end{align}
\end{theorem}
 
The detailed expressions of these bounds are provided in the Appendix. Here, we keep only those terms that emphasize the order of $T$. Note that our result outperforms the state-of-the-art model-free tabular result in average-reward CMDP \cite{wei2022provably}. However, our regret bound is worse than that achievable in average reward unconstrained MDP with general parameterization \cite{bai2023regret}. Interestingly, the gap between the convergence results of constrained and unconstrained setups is a common observation across the literature. For example, in the tabular model-free average reward MDP, the state-of-the-art regret bound for unconstrained setup, $\tilde{\mathcal{O}}(T^{1/2})$ \cite{wei2020model}, is better than that in the constrained setup, $\tilde{\mathcal{O}}(T^{5/6})$ \cite{wei2022provably}.


\section{Conclusions}

This paper establishes the first sublinear regret and constraint violation bounds in the average reward CMDP setup with general parametrization (and do not assume the underlying constrained Markov Decision Process (CMDP) to be tabular or linear). We show that our proposed algorithm achieves  $\tilde{\mathcal{O}}(T^{4/5})$ regret and constraint violation bounds where  $T$ is the time horizon. Note that the state of the art in unconstrained counterpart is $\tilde{\mathcal{O}}({T}^{3/4})$. Closing this gap by designing more efficient algorithms is an open question in the average reward CMDP literature with the general parametrization. Moreover, our current algorithm requires the knowledge of mixing time. Relaxing such assumptions is another important future direction in realistic settings. For further discussions on future work directions, the readers are referred to \cite{monograph}. 

\section{Acknowledgement}

This research was supported in part by the National Science Foundation under grant CCF-2149588 and Cisco, Inc.

\bibliography{ref}
\bibliographystyle{neurips}
\newpage
\appendix

\section{Proofs for Lemmas in Section \ref{sec_method}}
\subsection{Proof of Lemma \ref{lemma_grad_compute}}
Since the first step of the proof works in the same way for functions $J_r$ and $J_c$, we use the generic notations $J_g, V_g, Q_g$ where $g=r,c$ and derive the following.
    \begin{equation}
    \label{eq_22}
        \begin{aligned}
        &\nabla_{\theta} V_g^{\pi_{\theta_k}}(s)
        =\nabla_{\theta}\bigg(\sum_{a}\pi_{\theta}(a|s)Q_g^{\pi_{\theta}}(s, a)\bigg)\\
        &=\sum_{a}\bigg(\nabla_{\theta}\pi_{\theta}(a|s)\bigg)Q_g^{\pi_{\theta}}(s, a)+\sum_{a}\pi_\theta(a|s)\nabla_{\theta} Q_g^{\pi_{\theta}}(s, a)\\
        &\overset{(a)}=\sum_{a}\pi_{\theta}(a|s)\bigg(\nabla_{\theta}\log\pi_{\theta}(a|s)\bigg)Q_g^{\pi_{\theta}}(s, a)+\sum_{a}\pi_\theta(a|s)\nabla_{\theta} \bigg(g(s, a)-J_g(\theta)+\sum_{s'}P(s'|s, a)V_g^{\pi_{\theta}}(s')\bigg)\\
        &=\sum_{a}\pi_\theta(a|s)\bigg(\nabla_{\theta}\log\pi_{\theta}(a|s)\bigg)Q_g^{\pi_{\theta}}(s, a)+\sum_{a}\pi_\theta(a|s) \bigg(\sum_{s'}P(s'|s,a)\nabla_{\theta} V_g^{\pi_{\theta}}(s')\bigg) - \nabla_{\theta}J_g(\theta)
        \end{aligned}
    \end{equation}
    where the step (a) is a consequence of $\nabla_{\theta}\log\pi_{\theta}=\frac{\nabla\pi_{\theta}}{\pi_{\theta}}$ and the Bellman equation. Multiplying both sides by $d^{\pi_{\theta}}(s)$, taking a sum over $s\in\mathcal{S}$, and rearranging the terms, we obtain the following.
    \begin{align}
        \begin{split}
        &\nabla_{\theta}J_g(\theta)=\sum_{s}d^{\pi_{\theta}}(s)\nabla_{\theta}J_g(\theta)\\
        &=\sum_{s}d^{\pi_{\theta}}(s)\sum_{a}\pi_\theta(a|s)\bigg(\nabla_{\theta}\log\pi_{\theta}(a|s)\bigg)Q_g^{\pi_{\theta}}(s, a)+\sum_{s}d^{\pi_{\theta}}(s)\sum_{a}\pi_\theta(a|s) \bigg(\sum_{s'}P(s'|s,a)\nabla_{\theta} V_g^{\pi_{\theta}}(s')\bigg)\\
        &- \sum_{s}d^{\pi_{\theta}}(s)\nabla_{\theta}V_g^{\pi_\theta}(s)\\
        &\overset{}{=}\mathbf{E}_{s\sim d^{\pi_\theta}, a\sim \pi_\theta(\cdot|s)}\bigg[Q_g^{\pi_{\theta}}(s, a)\nabla_{\theta}\log\pi_{\theta}(a|s)\bigg]+\sum_{s}d^{\pi_{\theta}}(s) \sum_{s'}P^{\pi_\theta}(s'|s)\nabla_{\theta} V_g^{\pi_{\theta}}(s') - \sum_{s}d^{\pi_{\theta}}(s)\nabla_{\theta}V_g^{\pi_\theta}(s)\\
        &\overset{(a)}{=}\mathbf{E}_{s\sim d^{\pi_\theta}, a\sim \pi_\theta(\cdot|s)}\bigg[Q_g^{\pi_{\theta}}(s, a)\nabla_{\theta}\log\pi_{\theta}(a|s)\bigg]+ \sum_{s'}d^{\pi_{\theta}}(s')\nabla_{\theta} V_g^{\pi_{\theta}}(s') - \sum_{s}d^{\pi_{\theta}}(s)\nabla_{\theta}V_g^{\pi_\theta}(s)\\
        &=\mathbf{E}_{s\sim d^{\pi_\theta}, a\sim \pi_\theta(\cdot|s)}\bigg[Q_g^{\pi_{\theta}}(s, a)\nabla_{\theta}\log\pi_{\theta}(a|s)\bigg]
        \end{split}
    \end{align}
    where $(a)$ uses the fact that $d^{\pi_\theta}$ is a stationary distribution. Note that,
    \begin{equation}
        \begin{aligned}
        &\mathbf{E}_{s\sim d^{\pi_{\theta}},a\sim\pi_{\theta}(\cdot|s)}\bigg[ V_g^{\pi_{\theta}}(s)\nabla\log\pi_{\theta}(a|s)\bigg]\\
        &=\mathbf{E}_{s\sim d^{\pi_{\theta}}}\left[ \sum_{a\in\mathcal{A}}V_g^{\pi_{\theta}}(s)\nabla_{\theta}\pi_{\theta}(a|s)\right]\\
        &=\mathbf{E}_{s\sim d^{\pi_{\theta}}}\bigg[ V_g^{\pi_{\theta}}(s)\nabla_{\theta}\left(\sum_{a\in\mathcal{A}}\pi_{\theta}(a|s)\right)\bigg]
        =\mathbf{E}_{s\sim d^{\pi_{\theta}}}\bigg[ V_g^{\pi_{\theta}}(s)\nabla_{\theta}(1)\bigg]=0
        \end{aligned}
    \end{equation}
    We can, therefore, replace the function $Q_g^{\pi_{\theta}}$ in the policy gradient with the advantage function $A_g^{\pi_{\theta}}(s, a)=Q_g^{\pi_{\theta}}(s, a)-V_g^{\pi_{\theta}}(s)$, $\forall (s, a)\in\mathcal{S}\times \mathcal{A}$. Thus,
    \begin{equation}
        \nabla_{\theta} J_g(\theta)=\mathbf{E}_{s\sim d^{\pi_{\theta}},a\sim\pi_{\theta}(\cdot|s)}\bigg[ A_g^{\pi_{\theta}}(s,a)\nabla_{\theta}\log\pi_{\theta}(a|s)\bigg]
    \end{equation}
 The proof is completed using the definitions of $J_{\mathrm{L},\lambda}$ and $A_{\mathrm{L},\lambda}$.

\subsection{Proof of Lemma \ref{lemma_good_estimator}}

\begin{proof}
    The proof is similar to the proof of \cite[Lemma 6]{wei2020model}. Consider the $k$th epoch and assume that $\pi_{\theta_k}$ is denoted as $\pi$ for notational convenience. Let, $M$ be the number of disjoint sub-trajectories of length $N$ that start with the state $s$ and are at least $N$ distance apart (found by Algorithm \ref{alg:estQ}). Let, $g_{k, i}$ be the sum of rewards or constraint ($g=r, c$ accordingly) observed in the $i$th sub-trajectory and $\tau_i$ denote its starting time. The advantage function estimate is,
    \begin{align}
    \label{def_A_hat_appndx}
        \hat{A}_g^{\pi}(s, a) = \begin{cases}
            \dfrac{1}{\pi(a|s)}\left[\dfrac{1}{M}\sum_{i=1}^M g_{k,i}\mathrm{1}(a_{\tau_i}=a)\right] - \dfrac{1}{M}\sum_{i=1}^M g_{k,i}~&\text{if}~M>0\\
            0~&\text{if}~M=0
        \end{cases}
    \end{align}

    Note the following,
    \begin{align}
        \begin{split}
           &\mathbf{E}\left[g_{k,i}\bigg|s_{\tau_i}=s, a_{\tau_i}=a\right]
           =g(s, a) + \mathbf{E}\left[\sum_{t=\tau_i+1}^{\tau_i+N}g(s_t, a_t)\bigg| s_{\tau_i}=s, a_{\tau_i}=a\right]\\
           &=g(s, a) + \sum_{s'}P(s'|s, a)\mathbf{E}\left[\sum_{t=\tau_i+1}^{\tau_i+N}g(s_t, a_t)\bigg| s_{\tau_i+1}=s'\right]\\
           &=g(s, a) + \sum_{s'}P(s'|s, a)\left[\sum_{j=0}^{N-1}(P^{\pi})^j(s', \cdot)\right]^Tg^{\pi}\\
           &=g(s, a) + \sum_{s'}P(s'|s, a)\left[\sum_{j=0}^{N-1}(P^{\pi})^j(s', \cdot)-d^{\pi}\right]^Tg^{\pi} + N(d^{\pi})^Tg^{\pi}\\
           &\overset{(a)}{=}g(s, a) + \sum_{s'}P(s'|s, a)\left[\sum_{j=0}^{\infty}(P^{\pi})^j(s', \cdot)-d^{\pi}\right]^Tg^{\pi} + NJ_g^{\pi}-\underbrace{\sum_{s'}P(s'|s, a)\left[\sum_{j=N}^{\infty}(P^{\pi})^j(s', \cdot)-d^{\pi}\right]^Tg^{\pi}}_{\triangleq \mathrm{E}^{\pi}_T(s, a)}\\
           &\overset{(b)}{=} g(s, a) + \sum_{s'}P(s'|s, a)V_g^{\pi}(s') + NJ_g^{\pi}-\mathrm{E}^{\pi}_T(s, a)
           \overset{(c)}{=} Q_g^{\pi}(s, a) + (N+1)J_g^{\pi} - \mathrm{E}^{\pi}_T(s, a)
        \end{split}
    \end{align}
    where $(a)$ follows from the definition of $J_g^{\pi}$ as given in $(\ref{eq_r_pi_theta})$, $(b)$ is an application of the definition of $V_g^{\pi}$ given in $(\ref{def_v_pi_theta_s})$, and $(c)$ follows from the Bellman equation. Define the following quantity.
    \begin{align}
        \label{def_error_1}
        \delta^{\pi}(s, T) \triangleq \sum_{t=N}^{\infty}\norm{(P^{\pi})^t({s,\cdot}) - d^{\pi}}_1 ~~\text{where} ~N=4t_{\mathrm{mix}}(\log_2 T)
    \end{align}
    
    Using Lemma \ref{lemma_aux_3}, we get $\delta^{\pi}(s, T)\leq \frac{1}{T^3}$ which implies, $|\mathrm{E}^{\pi}_T(s, a)|\leq \frac{1}{T^3}$. Observe that,
    \begin{align}
    \label{eq_appndx_47}
        \begin{split}
            &\mathbf{E}\left[\left(\dfrac{1}{\pi(a|s)}g_{k,i}\mathrm{1}(a_{\tau_i}=a) - g_{k,i}\right)\bigg| s_{\tau_i}=s\right] \\
            &= \mathbf{E}\left[g_{k,i}\bigg| s_{\tau_i}=s, a_{\tau_i}=a\right] - \sum_{a'}\pi(a'|s)\mathbf{E}\left[g_{k,i}\bigg| s_{\tau_i}=s, a_{\tau_i}=a'\right]\\
            &=Q_g^{\pi}(s, a) + (N+1)J_g^{\pi} - \mathrm{E}^{\pi}_T(s, a) - \sum_{a'}\pi(a'|s)[Q^{\pi}(s, a) + (N+1)J_g^{\pi} - \mathrm{E}^{\pi}_T(s, a)]\\
            &=Q_g^{\pi}(s, a) - V_g^{\pi}(s)-\left[\mathrm{E}_T(s, a) - \sum_{a'}\pi(a'|s)\mathrm{E}_T^{\pi}(s, a')\right]\\
            &= A_g^{\pi}(s, a) -\Delta^{\pi}_T(s, a)
        \end{split}
    \end{align}
    where $\Delta^{\pi}_T(s, a)\triangleq\mathrm{E}_T(s, a) - \sum_{a'}\pi(a'|s)\mathrm{E}_T^{\pi}(s, a')$. Using the bound on $\mathrm{E}^{\pi}_T(s, a)$, we derive, $|\Delta_T^{\pi}(s, a)|\leq \frac{2}{T^3}$, which implies,
    \begin{align}
        \label{eq_appndx_48}\left|\mathbf{E}\left[\left(\dfrac{1}{\pi(a|s)}g_{k,i}\mathrm{1}(a_{\tau_i}=a) - g_{k,i}\right)\bigg| s_{\tau_i}=s\right] - A_g^{\pi}(s, a)\right|\leq |\Delta_T^{\pi}(s, a)|\leq\dfrac{2}{T^3}
    \end{align}

    Note that \eqref{eq_appndx_48} cannot be directly used to bound the bias of $\hat{A}_g^{\pi}(s, a)$. This is because the random variable $M$ is correlated with the  variables $\{g_{k,i}\}_{i=1}^M$. To decorrelate them, imagine a CMDP where the state distribution resets to the stationary distribution, $d^{\pi}$ after exactly $N$ time steps since the completion of a sub-trajectory. In other words, if a sub-trajectory starts at $\tau_{i}$, and ends at $\tau_i+N$, then the system `rests' for additional $N$ steps before rejuvenating with the state distribution, $d^{\pi}$ at $\tau_i+2N$. Clearly, the wait time between the reset after the $(i-1)$th sub-trajectory and the start of the $i$th sub-trajectory is, $w_{i}=\tau_{i}-(\tau_{i-1}+2N)$, $i>1$. Let $w_1$ be the difference between the start time of the $k$th epoch and the start time of the first sub-trajectory. Note that,

    $(a)$ $w_1$ only depends on the initial state, $s_{(k-1)H}$ and the induced transition function, $P^{\pi}$,

    $(b)$ $w_i$, where $i>1$, depends on the stationary distribution, $d^{\pi}$, and the induced transition function, $P^{\pi}$,

    $(c)$ $M$ only depends on $\{w_1, w_2, \cdots\}$ as other segments of the epoch have fixed length, $2N$.

    Clearly, in this imaginary CMDP, the sequence, $\{w_1, w_2, \cdots\}$, and hence, $M$ is independent of $\{g_{k,1}, g_{k, 2}, \cdots\}$. Let, $\mathbf{E}'$ denote the expectation operation and $\mathrm{Pr}'$ denote the probability of events in this imaginary system. Define the following.
    \begin{align}
    \label{def_delta_i}
        \Delta_i \triangleq \dfrac{g_{k,i}\mathrm{1}(a_{\tau_i}=a)}{\pi(a|s)} - g_{k,i} - A_g^{\pi}(s, a) + \Delta^{\pi}_T(s, a)
    \end{align}
    where $\Delta^{\pi}_T(s, a)$ is given in $(\ref{eq_appndx_47})$. Note that we have suppressed the dependence on $T$, $s, a$, and $\pi$ while defining $\Delta_i$ to remove clutter. Using $(\ref{eq_appndx_47})$, one can write $ \mathbf{E}'\left[\Delta_i(s, a)|\{w_i\}\right]=0$. Moreover, 
    \begin{align}
    \label{eq_appndx_50}
       \begin{split}
           &\mathbf{E}'\left[\left(\hat{A}_g^{\pi}(s, a) - A_g^{\pi}(s, a)\right)^2\right]\\ 
           &= \mathbf{E}'\left[\left(\hat{A}_g^{\pi}(s, a) - A_g^{\pi}(s, a)\right)^2\bigg| M>0\right]\times \mathrm{Pr}'(M>0) + \left(A_g^{\pi}(s, a)\right)^2\times \mathrm{Pr}'(M=0)\\
           &= \mathbf{E}'\left[\left(\dfrac{1}{M}\sum_{i=1}^M\Delta_i - \Delta_T^{\pi}(s, a)\right)^2\bigg| M>0\right]\times \mathrm{Pr}'(M>0) + \left(A_g^{\pi}(s, a)\right)^2\times \mathrm{Pr}'(M=0)\\
           & \overset{}{\leq} 2\mathbf{E}_{\{w_i\}}'\left[\mathbf{E}'\left[\left(\dfrac{1}{M}\sum_{i=1}^M\Delta_i \right)^2\bigg| \{w_i\}\right]\bigg| w_1\leq H-N\right]\times \mathrm{Pr}'(w_1\leq H-N) \\
           &+ 2\left(\Delta_T^{\pi}(s, a)\right)^2+\left(A_g^{\pi}(s, a)\right)^2\times \mathrm{Pr}'(M=0)\\
           & \overset{(a)}{\leq} 2\mathbf{E}_{\{w_i\}}'\left[\dfrac{1}{M^2}\sum_{i=1}^M \mathbf{E}'\left[\Delta_i^2\big|\{w_i\}\right]\bigg| w_1\leq H-N\right]\times \mathrm{Pr}'(w_1\leq H-N) \\
           &+ \dfrac{8}{T^6} +\left(A_g^{\pi}(s, a)\right)^2\times \mathrm{Pr}'(M=0)\\
       \end{split}
    \end{align}
    where $(a)$ uses the bound $|\Delta_T^{\pi}(s, a)|\leq \frac{2}{T^3}$ derived in $(\ref{eq_appndx_48})$, and the fact that $\{\Delta_i\}$ are zero mean independent random variables conditioned on $\{w_i\}$. Note that $|g_{k,i}|\leq N$ almost surely, $|A_g^{\pi}(s, a)|\leq \mathcal{O}(t_{\mathrm{mix}})$ via Lemma \ref{lemma_aux_2}, and $|\Delta^{\pi}_T(s, a)|\leq \frac{2}{T^3}$ as shown in $(\ref{eq_appndx_48})$. Combining, we get, $\mathbf{E}'[|\Delta_i|^2\big|\{w_i\}]\leq \mathcal{O}(N^2/\pi(a|s))$ (see the definition of $\Delta_i$ in (\ref{def_delta_i})). Invoking this bound into $(\ref{eq_appndx_50})$, we get the following result.
    \begin{align}
    \label{eq_appndx_51_}
        \begin{split}
        \mathbf{E}'&\left[\left(\hat{A}_g^{\pi}(s, a) - A_g^{\pi}(s, a)\right)^2\right]\leq 2\mathbf{E}'\left[\dfrac{1}{M}\bigg|w_1\leq H-N\right]\mathcal{O}\left(\dfrac{N^2}{\pi(a|s)}\right)+\dfrac{8}{T^6}\\
        &+\mathcal{O}(t_{\mathrm{mix}}^2)\times \mathrm{Pr}'(w_1>H-N)\\
        \end{split}
    \end{align}

    Note that, one can use Lemma \ref{lemma_aux_4} to bound the following violation probability.
    \begin{align}
    \label{eq_appndx_52_}
        \mathrm{Pr}'(w_1>H-N)\leq \left(1-\dfrac{3d^{\pi}(s)}{4}\right)^{4t_{\mathrm{hit}}T^{\xi}(\log T)-1}\overset{(a)}{\leq} \left(1-\dfrac{3d^{\pi}(s)}{4}\right)^{\dfrac{4}{d^{\pi}(s)}(\log T)}\leq \dfrac{1}{T^3}
    \end{align}
    where $(a)$ is a consequence of the fact that $4t_{\mathrm{hit}}T^{\xi}(\log_2 T) - 1 \geq \frac{4}{d^{\pi}(s)}\log_2 T$ for sufficiently large $T$. Finally, note that, if $M<M_0$, where $M_0$ is defined as,
    \begin{align}
        M_0\triangleq \dfrac{H-N}{2N+ \dfrac{4N\log T}{d^{\pi}(s)}}
    \end{align}
    then there exists at least one $w_i$ that exceeds $4N\log_2 T/d^{\pi}(s)$ which can happen with the following maximum probability according to Lemma \ref{lemma_aux_4}.
    \begin{align}
        \mathrm{Pr}'\left(M<M_0\right) \leq \left(1-\dfrac{3d^{\pi}(s)}{4}\right)^{\frac{4\log T}{d^{\pi(s)}}}\leq \dfrac{1}{T^3}
    \end{align}

    The above probability bound can be used to obtain the following result,
    \begin{align}
    \label{eq_appndx_55_}
    \begin{split}
            \mathbf{E}'\left[\dfrac{1}{M}\bigg| M>0\right]=\dfrac{\sum_{m=1}^{\infty}\dfrac{1}{m}\mathrm{Pr}'(M=m)}{\mathrm{Pr}'(M>0)}&\leq \dfrac{1\times \mathrm{Pr}'(M\leq M_0)+\dfrac{1}{M_0}\mathrm{Pr}'(M>M_0)}{\mathrm{Pr}'(M>0)}\\
        &\leq  \dfrac{\dfrac{1}{T^3}+\dfrac{2N+\dfrac{4N \log T}{d^{\pi}(s)}}{H-N}}{1-\dfrac{1}{T^3}}\leq \mathcal{O}\left(\dfrac{N\log T}{H d^{\pi}(s)}\right)
    \end{split}
    \end{align}

    Injecting $(\ref{eq_appndx_52_})$ and $(\ref{eq_appndx_55_})$ into $(\ref{eq_appndx_51_})$, we finally obtain the following.
    \begin{align}
        \label{eq_appndx_56_}
        \begin{split}
            \mathbf{E}'\left[\left(\hat{A}_g^{\pi}(s, a) - A_g^{\pi}(s, a)\right)^2\right]&\leq \mathcal{O}\left(\dfrac{N^3\log T}{H d^{\pi}(s)\pi(a|s)}\right)\\
            &=\mathcal{O}\left(\dfrac{N^3t_{\mathrm{hit}}\log T}{H \pi(a|s)}\right)=\mathcal{O}\left(\dfrac{t^2_{\mathrm{mix}}(\log T)^2}{T^{\xi}\pi(a|s)}\right)
        \end{split}
    \end{align}

    Eq. $(\ref{eq_appndx_56_})$ demonstrates that our desired inequality is obeyed in the imaginary system. We now need a mechanism to translate this result to our actual CMDP. Note that $(\hat{A}_g^{\pi}(s, a)-A_g^{\pi}(s, a))^2=f(X)$ where $X=(M, \tau_1, \mathcal{T}_1, \cdots, \tau_M, \mathcal{T}_M)$, and $\mathcal{T}_i = (a_{\tau_i}, s_{\tau_i+1}, a_{\tau_i+1}, \cdots, s_{\tau_i+N}, a_{\tau_i+N})$. We have,
    \begin{align}
    \label{eq_appndx_57_}
        \dfrac{\mathbf{E}[f(X)]}{\mathbf{E}'[f(X)]} = \dfrac{\sum_{X} f(X)\mathrm{Pr}(X)}{\sum_{X} f(X)\mathrm{Pr}'(X)}\leq \max_{X}\dfrac{\mathrm{Pr}(X)}{\mathrm{Pr'}(X)}
    \end{align}

    The last inequality uses the non-negativity of $f(\cdot)$. Observe that, for a fixed sequence, $X$, we have,
    \begin{align}
    \begin{split}
        \mathrm{Pr}(X) &= \mathrm{Pr}(\tau_1)\times \mathrm{Pr}(\mathcal{T}_1|\tau_1)\times \mathrm{Pr}(\tau_2|\tau_1, \mathcal{T}_1)\times \mathrm{Pr}(\mathcal{T}_2|\tau_2) \times \cdots \\
        &\times \mathrm{Pr}(\tau_M|\tau_{M-1}, \mathcal{T}_{M-1})\times \mathrm{Pr}(\mathcal{T}_M|\tau_M)\times \mathrm{Pr}(s_t\neq s, \forall t\in[\tau_M+2N, kH-N]|\tau_M, \mathcal{T}_M),
    \end{split}\\
        \begin{split}
            \mathrm{Pr}'(X) &= \mathrm{Pr}(\tau_1)\times \mathrm{Pr}(\mathcal{T}_1|\tau_1)\times \mathrm{Pr}'(\tau_2|\tau_1, \mathcal{T}_1)\times \mathrm{Pr}(\mathcal{T}_2|\tau_2) \times \cdots \\
        &\times \mathrm{Pr}'(\tau_M|\tau_{M-1}, \mathcal{T}_{M-1})\times \mathrm{Pr}(\mathcal{T}_M|\tau_M)\times \mathrm{Pr}(s_t\neq s, \forall t\in[\tau_M+2N, kH-N]|\tau_M, \mathcal{T}_M),
    \end{split}
    \end{align}

    The difference between $\mathrm{Pr}(X)$ and $\mathrm{Pr}'(X)$ arises because $\mathrm{Pr}(\tau_{i+1}|\tau_i, \mathcal{T}_i)\neq \mathrm{Pr}'(\tau_{i+1}|\tau_i, \mathcal{T}_i)$, $\forall i\in\{1, \cdots, M-1\}$. Note that the ratio of these two terms can be bounded as follows,
    \begin{align}
    \begin{split}
        &\dfrac{\mathrm{Pr}(\tau_{i+1}|\tau_i, \mathcal{T}_i)}{\mathrm{Pr}'(\tau_{i+1}|\tau_i, \mathcal{T}_i)}\\
        &=\dfrac{\sum_{s'\neq s} \mathrm{Pr}(s_{\tau_i+2N}=s'|\tau_i, \mathcal{T}_i)\times \mathrm{Pr}(s_t\neq s, \forall t\in [\tau_i+2N, \tau_{i+1}-1], s_{\tau_{i+1}}=s|s_{\tau_i+2N}=s')}{\sum_{s'\neq s} \mathrm{Pr}'(s_{\tau_i+2N}=s'|\tau_i, \mathcal{T}_i)\times \mathrm{Pr}(s_t\neq s, \forall t\in [\tau_i+2N, \tau_{i+1}-1], s_{\tau_{i+1}}=s|s_{\tau_i+2N}=s')}\\
        &\leq \max_{s'}\dfrac{\mathrm{Pr}(s_{\tau_i+2N}=s'|\tau_i, \mathcal{T}_i)}{\mathrm{Pr}'(s_{\tau_i+2N}=s'|\tau_i, \mathcal{T}_i)}\\
        &=\max_{s'}1+\dfrac{\mathrm{Pr}(s_{\tau_i+2N}=s'|\tau_i, \mathcal{T}_i)-d^{\pi}(s')}{d^{\pi}(s')}\overset{(a)}{\leq} \max_{s'}1+\dfrac{1}{T^3d^{\pi}(s')}\leq 1+\dfrac{t_{\mathrm{hit}}}{T^3}\leq 1+\dfrac{1}{T^2}
    \end{split}
    \end{align}
    where $(a)$ is a consequence of Lemma \ref{lemma_aux_3}. We have,
    \begin{align}
    \label{eq_appndx_61_}
        \dfrac{\mathrm{Pr}(X)}{\mathrm{Pr}'(X)}\leq \left(1+\dfrac{1}{T^2}\right)^M\leq e^{\frac{M}{T^2}}\overset{(a)}{\leq} e^{\frac{1}{T}}\leq \mathcal{O}\left(1+\dfrac{1}{T}\right) 
    \end{align}
    where $(a)$ uses the fact that $M\leq T$. Combining $(\ref{eq_appndx_57_})$ and $(\ref{eq_appndx_61_})$, we get,
    \begin{align}
        \begin{split}
        \mathbf{E}\left[\left(\hat{A}_g^{\pi}(s, a) - A_g^{\pi}(s, a)\right)^2\right]&\leq \mathcal{O}\left(1+\dfrac{1}{T}\right)\mathbf{E}'\left[\left(\hat{A}_g^{\pi}(s, a) - A_g^{\pi}(s, a)\right)^2\right]\\
        &\overset{(a)}{\leq} \mathcal{O}\left(\dfrac{t^2_{\mathrm{mix}}(\log T)^2}{T^{\xi}\pi(a|s)}\right)
        \end{split}
    \end{align}
    where $(a)$ follows from $(\ref{eq_appndx_56_})$. Using the definition of $A_{\mathrm{L},\lambda}$, we get,
        \begin{align*}
            &\mathbf{E}\left[\left(\hat{A}_{\mathrm{L},\lambda}^{\pi}(s, a) - A_{\mathrm{L},\lambda}^{\pi}(s, a)\right)^2\right]\\
            &=\mathbf{E}\left[\left((\hat{A}_r^{\pi}(s, a) - A_r^{\pi}(s, a))+\lambda(\hat{A}_c^{\pi}(s, a) - A_c^{\pi}(s, a))\right)^2\right]\\
            &\leq 2 \mathbf{E}\left[\left(\hat{A}_r^{\pi}(s, a) - A_r^{\pi}(s, a)\right)^2\right]+ 2\lambda^2\mathbf{E}\left[\left(\hat{A}_c^{\pi}(s, a) - A_c^{\pi}(s, a)\right)^2\right]\leq\mathcal{O}\left(\dfrac{t^2_{\mathrm{mix}}(\log T)^2}{\delta^2T^{\xi}\pi(a|s)}\right)
        \end{align*}
    This concludes the proof.
\end{proof}
\section{Proofs for the Section of Global Convergence Analysis}
\subsection{Proof of Lemma \ref{lemma_grad_est_bias}}
\begin{proof} 
Recall from Eq. \eqref{eq_grad_estimate} that,
\begin{align}
    &\omega_k= \dfrac{1}{H}\sum_{t=t_k}^{t_{k+1}-1}\hat{A}_{\mathrm{L},\lambda}^{\pi_{\theta_k}}(s_{t}, a_{t})\nabla_{\theta}\log \pi_{\theta_k}(a_{t}|s_{t}),
\end{align}

Define the following quantity,
\begin{align}
    &\bar{\omega}_k= \dfrac{1}{H}\sum_{t=t_k}^{t_{k+1}-1}A_{\mathrm{L},\lambda}^{\pi_{\theta_k}}(s_{t}, a_{t})\nabla_{\theta}\log \pi_{\theta_k}(a_{t}|s_{t})
\end{align}
where $t_k=(k-1)H$ is the starting time of the $k$th epoch. Note that the true gradient is given by, 
\begin{align}
    \nabla_{\theta}J_{\mathrm{L},\lambda}(\theta_k)=\mathbf{E}_{s\sim d^{\pi_{\theta_k}}, a\sim\pi_{\theta_k}(\cdot|s)}\left[A_{\mathrm{L},\lambda}^{\pi_{\theta_k}}(s, a)\nabla_{\theta}\log\pi_{\theta}(a|s)\right]
\end{align}

Using Assumption \ref{ass_score}, Lemma \ref{lemma_aux_2}, and $\lambda\in[0, \frac{2}{\delta}]$, one can exhibit that $|A_{\mathrm{L},\lambda}^{\pi_{\theta_k}}(s, a)\nabla_{\theta}\log\pi_{\theta}(a|s)|\leq \mathcal{O}(\frac{t_{\mathrm{mix}}G}{\delta})$, $\forall (s, a)\in \mathcal{S}\times \mathcal{A}$ which implies    $|\nabla_{\theta}J_{\mathrm{L},\lambda}(\theta_k)|\leq \mathcal{O}(\frac{t_{\mathrm{mix}}G}{\delta})$. Applying Lemma \ref{lemma_aux_7}, one, therefore, arrives at
\begin{align}
\label{eq_appndx_67_}
    \mathbf{E}\left[\norm{\bar{\omega}_k-\nabla_{\theta}J_{L,\lambda}(\theta_k)}^2\right]\leq \mathcal{O}\left(\frac{1}{\delta^2}G^{2}t^2_{\mathrm{mix}}\log T\right)\times \mathcal{O}\left(\dfrac{t_{\mathrm{mix}}\log T}{H}\right)=\mathcal{O}\left(\dfrac{G^2t_{\mathrm{mix}}^2}{\delta^2 t_{\mathrm{hit}}T^{\xi}}\right)
\end{align}

Finally, the difference, $\mathbf{E}\norm{\omega_k-\bar{\omega}_k}^2$ can be bounded as follows.
\begin{align}
    \label{eq_appndx_68_}
    \begin{split}
        &\mathbf{E}\norm{\omega_k-\bar{\omega}_k}^2\\
        &=\mathbf{E}\left[\left\|\dfrac{1}{H}\sum_{t=t_k}^{t_{k+1}-1}\hat{A}_{\mathrm{L},\lambda}^{\pi_{\theta_k}}(s_{t}, a_{t})\nabla_{\theta}\log \pi_{\theta_k}(a_{t}|s_{t})-\dfrac{1}{H}\sum_{t=t_k}^{t_{k+1}-1}\hat{A}_{\mathrm{L},\lambda}^{\pi_{\theta_k}}(s_{t}, a_{t})\nabla_{\theta}\log \pi_{\theta_k}(a_{t}|s_{t})\right\|^2\right]\\
        &\overset{(a)}{\leq} \dfrac{G^2}{H}\sum_{t=t_k}^{t_{k+1}-1}\mathbf{E}\left[\left(\hat{A}_{\mathrm{L},\lambda}^{\pi_{\theta_k}}(s_t, a_t)-A_{\mathrm{L},\lambda}^{\pi_{\theta_k}}(s_t, a_t)\right)^2\right]\\
        &\leq \dfrac{G^2}{H}\sum_{t=t_k}^{t_{k+1}-1}\mathbf{E}\left[\sum_{a}\pi_{\theta_k}(a|s_t)\mathbf{E}\left[\left(\hat{A}_{\mathrm{L},\lambda}^{\pi_{\theta_k}}(s_t, a)-A_{\mathrm{L},\lambda}^{\pi_{\theta_k}}(s_t, a)\right)^2\bigg| s_t\right]\right]\overset{(b)}{\leq}\mathcal{O}\left(\dfrac{AG^2t^2_{\mathrm{mix}}(\log T)^2}{\delta^2 T^{\xi}}\right)
    \end{split}
\end{align}
where $(a)$ follows from Assumption \ref{ass_score} and Jensen's inequality whereas $(b)$ follows from Lemma \ref{lemma_good_estimator}. Combining, $(\ref{eq_appndx_67_})$ and $(\ref{eq_appndx_68_})$, we conclude the result.
\end{proof}

\subsection{Proof of Lemma \ref{lem_performance_diff}}
\begin{proof}
Using the Lemma \ref{lemma_aux_5}, it is obvious to see that
    \begin{equation}
        \begin{aligned}
            J_g^{\pi}-J_g^{\pi'}&=\sum_{s}\sum_{a}d^{\pi}(s)(\pi(a|s)-\pi'(a|s))Q_g^{\pi'}(s,a)\\
            &=\sum_{s}\sum_{a}d^{\pi}(s)\pi(a|s)Q_g^{\pi'}(s,a)-\sum_{s}d^{\pi}(s)V_g^{\pi'}(s)\\
            &=\sum_{s}\sum_{a}d^{\pi}(s)\pi(a|s)Q_g^{\pi'}(s,a)-\sum_{s}\sum_{a}d^{\pi}(s)\pi(a|s)V_g^{\pi'}(s)\\
            &=\sum_{s}\sum_{a}d^{\pi}(s)\pi(a|s)[Q_g^{\pi'}(s,a)-V_g^{\pi'}(s)]=\mathbf{E}_{s\sim d^{\pi}}\mathbf{E}_{a\sim\pi(\cdot\vert s)}\big[A_g^{\pi'}(s,a)\big]
        \end{aligned}
    \end{equation}
    We conclude the lemma using the definition of $J_{\mathrm{L},\lambda}$ and $A_{\mathrm{L},\lambda}$.
\end{proof}

\subsection{Proof of Lemma \ref{lem_framework}}
\begin{proof}
    We start with the definition of KL divergence.
    \begin{equation}
	\begin{aligned}
            &\mathbf{E}_{s\sim d^{\pi^*}}[KL(\pi^*(\cdot\vert s)\Vert\pi_{\theta_k}(\cdot\vert s))-KL(\pi^*(\cdot\vert s)\Vert\pi_{\theta_{k+1}}(\cdot\vert s))]\\
            &=\mathbf{E}_{s\sim d^{\pi^*}}\mathbf{E}_{a\sim\pi^*(\cdot\vert s)}\bigg[\log\frac{\pi_{\theta_{k+1}(a\vert s)}}{\pi_{\theta_k}(a\vert s)}\bigg]\\
            &\overset{(a)}\geq\mathbf{E}_{s\sim d^{\pi^*}}\mathbf{E}_{a\sim\pi^*(\cdot\vert s)}[\nabla_\theta\log\pi_{\theta_k}(a\vert s)\cdot(\theta_{k+1}-\theta_k)]-\frac{B}{2}\Vert\theta_{k+1}-\theta_k\Vert^2\\
            &=\alpha\mathbf{E}_{s\sim d^{\pi^*}}\mathbf{E}_{a\sim\pi^*(\cdot\vert s)}[\nabla_{\theta}\log\pi_{\theta_k}(a\vert s)\cdot\omega_k]-\frac{B\alpha^2}{2}\Vert\omega_k\Vert^2\\
            &=\alpha\mathbf{E}_{s\sim d^{\pi^*}}\mathbf{E}_{a\sim\pi^*(\cdot\vert s)}[\nabla_\theta\log\pi_{\theta_k}(a\vert s)\cdot\omega^*_k]+\alpha\mathbf{E}_{s\sim d^{\pi^*}}\mathbf{E}_{a\sim\pi^*(\cdot\vert s)}[\nabla_\theta\log\pi_{\theta_k}(a\vert s)\cdot(\omega_k-\omega^*_k)]-\frac{B\alpha^2}{2}\Vert\omega_k\Vert^2\\
            &=\alpha[J_{\mathrm{L}}(\pi^*,\lambda_k)-J_{\mathrm{L}}(\theta_k,\lambda_k)]+\alpha\mathbf{E}_{s\sim d^{\pi^*}}\mathbf{E}_{a\sim\pi^*(\cdot\vert s)}[\nabla_\theta\log\pi_{\theta_k}(a\vert s)\cdot\omega^*_k]-\alpha[J_{\mathrm{L}}(\pi^*,\lambda_k)-J_{\mathrm{L}}(\theta_k,\lambda_k)]\\
            &+\alpha\mathbf{E}_{s\sim d^{\pi^*}}\mathbf{E}_{a\sim\pi^*(\cdot\vert s)}[\nabla_\theta\log\pi_{\theta_k}(a\vert s)\cdot(\omega_k-\omega^*_k)]-\frac{B\alpha^2}{2}\Vert\omega_k\Vert^2\\		
            &\overset{(b)}=\alpha[J_{\mathrm{L}}(\pi^*,\lambda_k)-J_{\mathrm{L}}(\theta_k,\lambda_k)]+\alpha\mathbf{E}_{s\sim d^{\pi^*}}\mathbf{E}_{a\sim\pi^*(\cdot\vert s)}\bigg[\nabla_\theta\log\pi_{\theta_k}(a\vert s)\cdot\omega^*_k-A_{\mathrm{L},\lambda_k}^{\pi_{\theta_k}}(s,a)\bigg]\\
            &+\alpha\mathbf{E}_{s\sim d^{\pi^*}}\mathbf{E}_{a\sim\pi^*(\cdot\vert s)}[\nabla_\theta\log\pi_{\theta_k}(a\vert s)\cdot(\omega_k-\omega^*_k)]-\frac{B\alpha^2}{2}\Vert\omega_k\Vert^2\\
            &\overset{(c)}\geq\alpha[J_{\mathrm{L}}(\pi^*,\lambda_k)-J_{\mathrm{L}}(\theta_k,\lambda_k)]-\alpha\sqrt{\mathbf{E}_{s\sim d^{\pi^*}}\mathbf{E}_{a\sim\pi^*(\cdot\vert s)}\bigg[\bigg(\nabla_\theta\log\pi_{\theta_k}(a\vert s)\cdot\omega^*_k-A_{\mathrm{L},\lambda_k}^{\pi_{\theta_k}}(s,a)\bigg)^2\bigg]}\\
            &-\alpha\mathbf{E}_{s\sim d^{\pi^*}}\mathbf{E}_{a\sim\pi^*(\cdot\vert s)}\Vert\nabla_\theta\log\pi_{\theta_k}(a\vert s)\Vert_2\Vert(\omega_k-\omega^*_k)\Vert-\frac{B\alpha^2}{2}\Vert\omega_k\Vert^2\\
            &\overset{(d)}\geq\alpha[J_{\mathrm{L}}(\pi^*,\lambda_k)-J_{\mathrm{L}}(\theta_k,\lambda_k)]-\alpha\sqrt{\epsilon_{\mathrm{bias}}}-\alpha G\Vert(\omega_k-\omega^*_k)\Vert-\frac{B\alpha^2}{2}\Vert\omega_k\Vert^2\\
	\end{aligned}	
    \end{equation}
    where the step (a) holds by Assumption \ref{ass_score} and step (b) holds by Lemma \ref{lem_performance_diff}. Step (c) uses the convexity of the function $f(x)=x^2$. Finally, step (d) comes from the Assumption \ref{ass_transfer_error}. Rearranging items, we have
    \begin{equation}
	\begin{split}
            J_{\mathrm{L}}(\pi^*,\lambda_k)-J_{\mathrm{L}}(\theta_k,\lambda_k)&\leq \sqrt{\epsilon_{\mathrm{bias}}}+ G\Vert(\omega_k-\omega^*_k)\Vert+\frac{B\alpha}{2}\Vert\omega_k\Vert^2\\
            &+\frac{1}{\alpha}\mathbf{E}_{s\sim d^{\pi^*}}[KL(\pi^*(\cdot\vert s)\Vert\pi_{\theta_k}(\cdot\vert s))-KL(\pi^*(\cdot\vert s)\Vert\pi_{\theta_{k+1}}(\cdot\vert s))]
	\end{split}
    \end{equation}
    Summing from $k=1$ to $K$, using the non-negativity of KL divergence and dividing the resulting expression by $K$, we get the desired result.
\end{proof}

\subsection{Proof of Lemma \ref{lemma:41ss}}

\begin{proof}
    By the $L$-smooth property of the objective function and constraint function, we know that $J_{\mathrm{L}}(\cdot,\lambda)$ is a $L(1+\lambda)$-smooth function. Thus,
    \begin{align}
        \begin{split}
            &J_{\mathrm{L}}(\theta_{k+1},\lambda_k)
            \geq J_{\mathrm{L}}(\theta_k,\lambda_k)+\left<\nabla J_{\mathrm{L}}(\theta_k,\lambda_k),\theta_{k+1}-\theta_k\right>-\frac{L(1+\lambda_k)}{2}\Vert\theta_{k+1}-\theta_k\Vert^2\\
            &\overset{(a)} =J_{\mathrm{L}}(\theta_k,\lambda_k) + \alpha \nabla J_{\mathrm{L}}(\theta_k,\lambda_k)^T \omega_k - \frac{L(1+\lambda_k) \alpha^2}{2} \norm{ \omega_k }^2 \\
            &= J_{\mathrm{L}}(\theta_k,\lambda_k) + \alpha \norm{\nabla J_{\mathrm{L}}(\theta_k,\lambda_k)}^2 - \alpha \langle \nabla J_{\mathrm{L}}(\theta_k,\lambda_k) - \omega_k, \nabla J_{\mathrm{L}}(\theta_k, \lambda_k) \rangle \\
            &\quad- \frac{L(1+\lambda_k) \alpha^2}{2}\Vert \nabla J_{\mathrm{L}}(\theta_k,\lambda_k)-\omega_k-\nabla J_{\mathrm{L}}(\theta_k,\lambda_k)\Vert^2\\
            &\overset{(b)}\geq J_{\mathrm{L}}(\theta_k,\lambda_k) + \alpha \norm{\nabla J_{\mathrm{L}}(\theta_k,\lambda_k)}^2 - \frac{\alpha}{2} \Vert\nabla J_{\mathrm{L}}(\theta_k,\lambda_k) - \omega_k\Vert^2 -\frac{\alpha}{2}\Vert\nabla J_{\mathrm{L}}(\theta_k,\lambda_k)\Vert^2\\
            &\quad- L(1+\lambda_k)\alpha^2\Vert \nabla J_{\mathrm{L}}(\theta_k,\lambda_k)-\omega_k\Vert^2-L(1+\lambda_k)\alpha^2\Vert\nabla J_{\mathrm{L}}(\theta_k,\lambda_k)\Vert^2\\
            &= J_{\mathrm{L}}(\theta_k,\lambda_k) + \left(\frac{\alpha}{2}-L(1+\lambda_k)\alpha^2\right) \norm{\nabla J_{\mathrm{L}}(\theta_k,\lambda_k)}^2 - \left(\frac{\alpha}{2}+L(1+\lambda_k)\alpha^2\right) \Vert\nabla J_{\mathrm{L}}(\theta_k,\lambda_k) - \omega_k\Vert^2
        \end{split}
    \end{align}
    where step (a) follows from the fact that $\theta_{k+1} = \theta_k + \alpha \omega_k$ and inequality (b) holds due to the Cauchy-Schwarz inequality. Now, adding $J_{\mathrm{L}}(\theta_{k+1},\lambda_{k+1})$ on both sides, we have
    \begin{equation}
        \begin{aligned}
            J_{\mathrm{L}}(\theta_{k+1},\lambda_{k+1})&\geq J_{\mathrm{L}}(\theta_{k+1},\lambda_{k+1})- J_{\mathrm{L}}(\theta_{k+1},\lambda_{k}) +J_{\mathrm{L}}(\theta_k,\lambda_k) + \left(\frac{\alpha}{2}-L(1+\lambda_k)\alpha^2\right) \norm{\nabla J_{\mathrm{L}}(\theta_k,\lambda_k)}^2\\
            &\quad - \left(\frac{\alpha}{2}+L(1+\lambda_k)\alpha^2\right) \Vert\nabla J_{\mathrm{L}}(\theta_k,\lambda_k) - \omega_k\Vert^2\\
            &\overset{(a)}= (\lambda_{k+1}-\lambda_{k})J_c(\theta_{k+1}) +J_{\mathrm{L}}(\theta_k,\lambda_k) + \left(\frac{\alpha}{2}-L(1+\lambda_k)\alpha^2\right) \norm{\nabla J_{\mathrm{L}}(\theta_k,\lambda_k)}^2\\
            &\quad - \left(\frac{\alpha}{2}+L(1+\lambda_k)\alpha^2\right) \Vert\nabla J_{\mathrm{L}}(\theta_k,\lambda_k) - \omega_k\Vert^2\\
            &\overset{(b)}\geq -\beta +J_{\mathrm{L}}(\theta_k,\lambda_k) + \left(\frac{\alpha}{2}-L(1+\lambda_k)\alpha^2\right) \norm{\nabla J_{\mathrm{L}}(\theta_k,\lambda_k)}^2\\
            &\quad - \left(\frac{\alpha}{2}+L(1+\lambda_k)\alpha^2\right) \Vert\nabla J_{\mathrm{L}}(\theta_k,\lambda_k) - \omega_k\Vert^2\\
        \end{aligned}
    \end{equation}
    where (a) holds by the definition of $J_{\mathrm{L}}(\theta,\lambda)$ and step (b) is true because $|J_c(\theta)|\leq 1,\forall \theta$ and $|\lambda_{k+1}-\lambda_k|\leq \beta|\hat{J}_c(\theta_k)|\leq \beta$ where the last inequality uses the fact that $|\hat{J}_c(\theta_k)|\leq 1$. Summing over $k\in\{ 1, \cdots, K\}$, we have,
    \begin{equation}
        \begin{aligned}
            \sum_{k=1}^{K}\bigg[J_{\mathrm{L}}(\theta_{k+1},\lambda_{k+1})-J_{\mathrm{L}}(\theta_k,\lambda_k)\bigg]&\geq -\beta K  + \sum_{k=1}^{K}\left(\frac{\alpha}{2}-L(1+\lambda_k)\alpha^2\right) \norm{\nabla J_{\mathrm{L}}(\theta_k,\lambda_k)}^2\\
        &\quad - \sum_{k=1}^{K}\left(\frac{\alpha}{2}+L(1+\lambda_k)\alpha^2\right) \Vert\nabla J_{\mathrm{L}}(\theta_k,\lambda_k) - \omega_k\Vert^2\\
        \end{aligned}
    \end{equation}
    which leads to the following.
    \begin{equation}
        \begin{aligned}
        J_{\mathrm{L}}(\theta_{K+1},\lambda_{K+1})-J_{\mathrm{L}}(\theta_1,\lambda_1)&\geq -\beta K  + \sum_{k=1}^{K}\left(\frac{\alpha}{2}-L(1+\lambda_k)\alpha^2\right) \norm{\nabla J_{\mathrm{L}}(\theta_k,\lambda_k)}^2\\
        &\quad - \sum_{k=1}^{K}\left(\frac{\alpha}{2}+L(1+\lambda_k)\alpha^2\right) \Vert\nabla J_{\mathrm{L}}(\theta_k,\lambda_k) - \omega_k\Vert^2\\
        \end{aligned}
    \end{equation}
    Rearranging the terms and using $0\leq \lambda_k\leq \frac{2}{\delta}$ due to the dual update, we arrive at the following. 
    \begin{align}
    \begin{split}
        &\sum_{k=1}^{K}\norm{\nabla J_{\mathrm{L}}(\theta_k,\lambda_k)}^2 
        \leq \frac{J_{\mathrm{L}}(\theta_{K+1},\lambda_{K+1})-J_{\mathrm{L}}(\theta_1,\lambda_1) + \beta K + (\frac{\alpha}{2}+L(1+\frac{2}{\delta})\alpha^2) \sum_{k=1}^{K}  \Vert\nabla J_{\mathrm{L}}(\theta_k,\lambda_k) - \omega_k\Vert^2}{\frac{\alpha}{2}-L(1+\frac{2}{\delta})\alpha^2}
        \end{split}
    \end{align}
    Choosing $\alpha = \frac{1}{4L(1+\frac{2}{\delta})}$ and dividing both sides by $K$, we conclude the result.
    \begin{align}
    \begin{split}
        \frac{1}{K}\sum_{k=1}^{K}\norm{\nabla J_{\mathrm{L}}(\theta_k,\lambda_k)}^2 &\leq \frac{16L(1+\frac{2}{\delta})}{K}\left[J_{\mathrm{L}}(\theta_{K+1},\lambda_{K+1})-J_{\mathrm{L}}(\theta_1,\lambda_1)\right]\\
        &+  \dfrac{3}{K}\sum_{k=1}^{K}\Vert\nabla J_{\mathrm{L}}(\theta_k,\lambda_k) - \omega_k\Vert^2 + \beta
    \end{split}
    \end{align}
    Recall that $|J_{\mathrm{L}}(\theta,\lambda)|\leq 1+\lambda\leq 1+\frac{2}{\delta}\leq \frac{3}{\delta},\forall \theta\in \Theta,\forall\lambda\geq 0$. Thus,
    \begin{align}
        \frac{1}{K}\sum_{k=1}^{K}\norm{\nabla J_{\mathrm{L}}(\theta_k,\lambda_k)}^2 \leq \frac{288L}{\delta^2 K}+  \dfrac{3}{K}\sum_{k=1}^{K}\Vert\nabla J_{\mathrm{L}}(\theta_k,\lambda_k) - \omega_k\Vert^2 + \beta
    \end{align}
    This completes the proof.
\end{proof}

\subsection{Proof of Theorem \ref{thm_global_convergence}}\label{app_thm}

\subsubsection{Rate of Convergence of the Objective}

Recall the definition of $J_{\mathrm{L}}(\theta,\lambda)=J_r(\theta)+\lambda J_c(\theta)$. Using Lemma \ref{Lemma_global_Lagrange}, we arrive at the following.
\begin{equation}
    \label{eq:bound_Jr1}
    \begin{aligned}
        &\frac{1}{K}\sum_{k=1}^{K}\mathbf{E}\bigg(J_r^{\pi^*}-J_r(\theta_k)\bigg)\leq  G\left(1+\dfrac{1}{\mu_F}\right)\tilde{\mathcal{O}}\left(\sqrt{\beta}+\dfrac{\sqrt{A}Gt_{\mathrm{mix}}}{\delta T^{\xi/2}}+\dfrac{\sqrt{Lt_{\mathrm{mix}}t_{\mathrm{hit}}}}{\delta T^{(1-\xi)/2}}\right)\\
        &+\dfrac{B}{L}\tilde{\mathcal{O}}\left(\dfrac{AG^2t_{\mathrm{mix}}^2}{\delta^2 T^{\xi}}+\dfrac{Lt_{\mathrm{mix}}t_{\mathrm{hit}}}{\delta^2 T^{1-\xi}}+\beta\right) + \mathcal{\tilde{O}}\bigg(\frac{Lt_{\mathrm{mix}}t_{\mathrm{hit}}\mathbf{E}_{s\sim d^{\pi^*}}[KL(\pi^*(\cdot\vert s)\Vert\pi_{\theta_1}(\cdot\vert s))]}{T^{1-\xi}\delta }\bigg)\\
        &-\frac{1}{K}\sum_{k=1}^{K}\mathbf{E}\bigg[\lambda_k\bigg(J_{c}^{\pi^*}-J_c(\theta_k)\bigg)\bigg] + \sqrt{\epsilon_{\mathrm{bias}}} 
    \end{aligned}
\end{equation}

Thus, we need to find a bound for the last term in the above equation.
\begin{equation}
    \label{eq:bound_lambdak}
    \begin{aligned}
        0&\leq (\lambda_{K+1})^2\\
        &\overset{(a)}{=}\sum_{k=1}^{K}\bigg((\lambda_{k+1})^2-(\lambda_{k})^2\bigg)\\
        &=\sum_{k=1}^{K}\bigg(\mathcal{P}_{[0,\frac{2}{\delta}]}\big[\lambda_{k}-\beta\hat{J}_{c}(\theta_k)\big]^2-(\lambda_{k})^2\bigg)\\
        &\leq\sum_{k=1}^{K}\bigg(\big[\lambda_{k}-\beta\hat{J}_{c}(\theta_k)\big]^2-(\lambda_{k})^2\bigg)\\
        &=-2\beta\sum_{k=1}^{K}\lambda_{k}\hat{J}_{c}(\theta_k)+\beta^2\sum_{k=1}^{K}\hat{J}_{c}(\theta_k)^2\\
        &\overset{(b)}\leq 2\beta\sum_{k=1}^{K}\lambda_{k}(J_{c}^{\pi^*}- \hat{J}_{c}(\theta_k))+\beta^2\sum_{k=1}^{K}\hat{J}_{c}(\theta_k)^2\\
        &\leq 2\beta\sum_{k=1}^{K}\lambda_{k}(J_{c}^{\pi^*}- \hat{J}_{c}(\theta_k))+2\beta^2\sum_{k=1}^{K}\hat{J}_{c}(\theta_k)^2\\
        &= 2\beta\sum_{k=1}^{K}\lambda_{k}(J_{c}^{\pi^*}- J_{c}(\theta_k))+2\beta\sum_{k=1}^{K}\lambda_{k}(J_{c}(\theta_k)- \hat{J}_{c}(\theta_k))+2\beta^2\sum_{k=1}^{K}\hat{J}_{c}(\theta_k)^2
    \end{aligned}
\end{equation}
where (a) uses $\lambda_1=0$ and inequality (b) holds because $\theta^*$ is a feasible solution to the constrained optimization problem. Rearranging items and taking the expectation, we have,
\begin{align}
    \label{eq_appndx_71_new}
    \begin{split}
        -\frac{1}{K}\sum_{k=1}^{K}\mathbf{E}\bigg[\lambda_{k}(J_{c}^{\pi^*}- J_{c}(\theta_k))\bigg] &\leq \frac{1}{K}\sum_{k=1}^{K}\mathbf{E}\bigg[\lambda_{k}(J_{c}(\theta_k)- \hat{J}_{c}(\theta_k))\bigg]+\frac{\beta}{K}\sum_{k=1}^{K}\mathbf{E}[\hat{J}_{c}(\theta_k)]^2\\
        & \overset{(a)}{\leq} \frac{1}{K}\sum_{k=1}^{K}\mathbf{E}\left[\lambda_{k}\left(J_{c}(\theta_k)- \hat{J}_{c}(\theta_k)\right)\right]+\beta\\
        &\overset{(b)}{=} \frac{1}{K}\sum_{k=1}^{K}\mathbf{E}\bigg[\lambda_{k}\left(J_{c}(\theta_k)- \mathbf{E}\left[\hat{J}_{c}(\theta_k)\big|\theta_k\right]\right)\bigg]+\beta\\
        &\leq \frac{1}{K}\sum_{k=1}^{K}\mathbf{E}\bigg[\lambda_{k}\left|J_{c}(\theta_k)- \mathbf{E}\left[\hat{J}_{c}(\theta_k)\big|\theta_k\right]\right|\bigg]+\beta\overset{(c)}{\leq} \frac{2}{\delta T^2}+\beta
    \end{split}
\end{align}
where (a) results from $|\hat{J}_{c, \rho}(\theta)|^2\leq 1$, $\forall \theta\in\Theta$ and (b) uses the fact that $\hat{J}_{c, \rho}(\theta_k)$ and $\lambda_k$ are conditionally independent given $\theta_k$. Finally, (c) is a consequence of Lemma \ref{lemma_aux_6}. Combining \eqref{eq_appndx_71_new} with  \eqref{eq:bound_Jr1}, we deduce,
\begin{equation}
    \label{eq_bound_Jr_final}
    \begin{aligned}
       &\frac{1}{K}\sum_{k=1}^{K}\mathbf{E}\bigg(J_r^{\pi^*}-J_r(\theta_k)\bigg)\\
       &\leq  \sqrt{\epsilon_{\mathrm{bias}}} + G\left(1+\dfrac{1}{\mu_F}\right)\tilde{\mathcal{O}}\left(\sqrt{\beta}+\dfrac{\sqrt{A}Gt_{\mathrm{mix}}}{\delta T^{\xi/2}}+\dfrac{\sqrt{Lt_{\mathrm{mix}}t_{\mathrm{hit}}}}{\delta T^{(1-\xi)/2}}\right)+ \mathcal{O}\left(\dfrac{1}{\delta T^2}+\beta\right)\\
        &+\dfrac{B}{L}\tilde{\mathcal{O}}\left(\dfrac{AG^2t_{\mathrm{mix}}^2}{\delta^2 T^{\xi}}+\dfrac{Lt_{\mathrm{mix}}t_{\mathrm{hit}}}{\delta^2 T^{1-\xi}}+\beta\right) + \mathcal{\tilde{O}}\bigg(\frac{Lt_{\mathrm{mix}}t_{\mathrm{hit}}\mathbf{E}_{s\sim d^{\pi^*}}[KL(\pi^*(\cdot\vert s)\Vert\pi_{\theta_1}(\cdot\vert s))]}{T^{1-\xi}\delta }\bigg)\\
        &\leq \sqrt{\epsilon_{\mathrm{bias}}} + G\left(1+\dfrac{1}{\mu_F}\right)\tilde{\mathcal{O}}\left(\sqrt{\beta}+\dfrac{\sqrt{A}Gt_{\mathrm{mix}}}{\delta T^{\xi/2}}+\dfrac{\sqrt{Lt_{\mathrm{mix}}t_{\mathrm{hit}}}}{\delta T^{(1-\xi)/2}}\right)
    \end{aligned}
\end{equation}

The last inequality presents only the dominant terms of $\beta$ and $T$.
   
\subsubsection{Rate of Constraint Violation}
 Since $\{\lambda_k\}_{k=1}^{K}$ are derived by applying the dual update in Algorithm \ref{alg:PG_MAG}, we have,
\begin{equation}
    \label{eq_appndx_73_new}
    \begin{aligned}
	    &\mathbf{E}\left\vert\lambda_{k+1} - \dfrac{2}{\delta}\right\vert^2 
		\overset{(a)}{\leq} \mathbf{E}\left|\lambda_{k} - \beta \hat{J}_c(\theta_{k}) - \dfrac{2}{\delta}\right|^2\\
		&=\mathbf{E}\left|\lambda_{k} -\dfrac{2}{\delta}\right|^2 -2\beta \mathbf{E}\left[\hat{J}_c(\theta_k)\left(\lambda_{k}  -\dfrac{2}{\delta}\right)\right] +\beta^2 \mathbf{E}\left[\hat{J}^2_c(\theta_{k})\right]
		\\
		&\overset{(b)}\leq\mathbf{E}\left|\lambda_{k} -\dfrac{2}{\delta}\right|^2 - 2\beta \mathbf{E}\left[J_c(\theta_{k})\left(\lambda_{k} -\dfrac{2}{\delta}\right)\right]-2\beta \mathbf{E}\left[\left(\hat{J}_c(\theta_{k})-J_c(\theta_{k})\right)\left(\lambda_{k}-\dfrac{2}{\delta}\right)\right] + \beta^2\\
        &\overset{(c)}{=}\mathbf{E}\left|\lambda_{k} -\dfrac{2}{\delta}\right|^2 - 2\beta \mathbf{E}\left[J_c(\theta_{k})\left(\lambda_{k} -\dfrac{2}{\delta}\right)\right]-2\beta \mathbf{E}\left[\left(\mathbf{E}\left[\hat{J}_c(\theta_{k})\big|\theta_k\right]-J_c(\theta_{k})\right)\left(\lambda_{k}-\dfrac{2}{\delta}\right)\right] + \beta^2\\
        &\leq \mathbf{E}\left|\lambda_{k} -\dfrac{2}{\delta}\right|^2 - 2\beta \mathbf{E}\left[J_c(\theta_{k})\left(\lambda_{k} -\dfrac{2}{\delta}\right)\right]+2\beta \mathbf{E}\left[\left|\mathbf{E}\left[\hat{J}_c(\theta_{k})\big|\theta_k\right]-J_c(\theta_{k})\right|\left|\lambda_{k}-\dfrac{2}{\delta}\right|\right] + \beta^2\\
        &\overset{(d)}{\leq} \mathbf{E}\left|\lambda_{k} -\dfrac{2}{\delta}\right|^2 - 2\beta \mathbf{E}\left[J_c(\theta_{k})\left(\lambda_{k} -\dfrac{2}{\delta}\right)\right] + \dfrac{4\beta}{\delta T^2} + \beta^2
	\end{aligned}
\end{equation}
where $(a)$ is due to the non-expansiveness of the projection $\mathcal{P}_{[0, \frac{2}{\delta}]}$ and $(b)$ holds because $\hat{J}_c(\theta)\in[0,1]$, $\forall \theta\in\Theta$ according to its definition in Algorithm \ref{alg:PG_MAG}. Finally, $(c)$ is a consequence of the fact that $\hat{J}_c(\theta_k)$ and $\lambda_k$ are conditionally independent given $\theta_k$ whereas $(d)$ applies $|\lambda_k-\frac{2}{\delta}|\leq \frac{2}{\delta}$ and Lemma \ref{lemma_aux_6}. Averaging \eqref{eq_appndx_73_new} over $k\in\{1,\ldots,K\}$, we get,
\begin{equation}
    \begin{aligned}
    \frac{1}{K}\sum_{k=1}^{K} \mathbf{E}\left[J_c(\theta_k)\left(\lambda_{k} -\frac{2}{\delta}\right)\right] &\leq \frac{\left\vert\lambda_{1} - \frac{2}{\delta}\right\vert^2 -\left\vert\lambda_{K+1} - \frac{2}{\delta}\right\vert^2 }{2\beta K} + \dfrac{2}{\delta T^2} + \dfrac{\beta}{2}
    \overset{(a)}{\leq} \dfrac{2}{\delta^2 \beta K} + \dfrac{2}{\delta T^2} + \dfrac{\beta}{2}
    \end{aligned}
\end{equation}
where (a) uses $\lambda_1=0$. Note that $\lambda_k J_c^{\pi^*}\geq 0$, $\forall k$. Adding the above inequality to \eqref{eq:bound_Jr1} at both sides, we, therefore, have,
\begin{equation}
    \label{eq:bound_Jc2}
    \begin{aligned}
        &\mathbf{E}\bigg[J_r^{\pi^*}-\frac{1}{K}\sum_{k=1}^{K}J_r(\theta_k)\bigg]+\frac{2}{\delta}\mathbf{E}\bigg[\frac{1}{K}\sum_{k=1}^{K}-J_c(\theta_k)\bigg]\leq \sqrt{\epsilon_{\mathrm{bias}}} +\frac{2}{\delta^2\beta K} +\frac{2}{T^2\delta}+ \frac{\beta}{2}\\
        &+G\left(1+\dfrac{1}{\mu_F}\right)\tilde{\mathcal{O}}\left(\sqrt{\beta}+\dfrac{\sqrt{A}Gt_{\mathrm{mix}}}{\delta T^{\xi/2}}+\dfrac{\sqrt{Lt_{\mathrm{mix}}t_{\mathrm{hit}}}}{\delta T^{(1-\xi)/2}}\right) + \dfrac{B}{L}\tilde{\mathcal{O}}\left(\dfrac{AG^2t_{\mathrm{mix}}^2}{\delta^2 T^{\xi}}+\dfrac{Lt_{\mathrm{mix}}t_{\mathrm{hit}}}{\delta^2 T^{1-\xi}}+\beta\right)\\
        &+ \mathcal{\tilde{O}}\bigg(\frac{Lt_{\mathrm{mix}}t_{\mathrm{hit}}\mathbf{E}_{s\sim d^{\pi^*}}[KL(\pi^*(\cdot\vert s)\Vert\pi_{\theta_1}(\cdot\vert s))]}{T^{1-\xi}\delta }\bigg)
    \end{aligned}
\end{equation}
Since the functions $\{J_g(\theta_k)\}, k\in\{0, \cdots, K-1\}, g\in\{r, c\}$ are linear in occupancy measure, there exists a policy $\bar{\pi}$ such that the following holds $\forall g\in\{r, c\}$.
\begin{equation}\label{eq_avg_value}
	\frac{1}{K}\sum_{k=1}^{K}J_g(\theta_k)=J_g^{\bar\pi}
\end{equation}
Injecting the above relation to \eqref{eq:bound_Jc2}, we have
\begin{equation}
    \begin{aligned}
        &\mathbf{E}\bigg[J_r^{\pi^*}-J_r^{\bar\pi}\bigg]+\frac{2}{\delta}\mathbf{E}\bigg[-J_c^{\bar\pi}\bigg]\leq \sqrt{\epsilon_{\mathrm{bias}}}+\frac{2}{\delta^2\beta K} +\frac{2}{T^2\delta}+ \frac{\beta}{2} \\
        &+G\left(1+\dfrac{1}{\mu_F}\right)\tilde{\mathcal{O}}\left(\sqrt{\beta}+\dfrac{\sqrt{A}Gt_{\mathrm{mix}}}{\delta T^{\xi/2}}+\dfrac{\sqrt{Lt_{\mathrm{mix}}t_{\mathrm{hit}}}}{\delta T^{(1-\xi)/2}}\right)\\
        &+\dfrac{B}{L}\tilde{\mathcal{O}}\left(\dfrac{AG^2t_{\mathrm{mix}}^2}{\delta^2 T^{\xi}}+\dfrac{Lt_{\mathrm{mix}}t_{\mathrm{hit}}}{\delta^2 T^{1-\xi}}+\beta\right) + \mathcal{\tilde{O}}\bigg(\frac{Lt_{\mathrm{mix}}t_{\mathrm{hit}}\mathbf{E}_{s\sim d^{\pi^*}}[KL(\pi^*(\cdot\vert s)\Vert\pi_{\theta_1}(\cdot\vert s))]}{T^{1-\xi}\delta }\bigg)
    \end{aligned}
\end{equation}
By Lemma \ref{lem.constraint}, we arrive at,
\begin{equation}
\label{eq_83}
    \begin{aligned}
        &\mathbf{E}\bigg[-J_c^{\bar\pi}\bigg]\\
        &\leq  \delta\sqrt{\epsilon_{\mathrm{bias}}}+\frac{2}{\delta\beta K} +\frac{2}{T^2}+\frac{\delta\beta}{2}+G\left(1+\dfrac{1}{\mu_F}\right)\tilde{\mathcal{O}}\left(\delta\sqrt{\beta}+\dfrac{\sqrt{A}Gt_{\mathrm{mix}}}{ T^{\xi/2}}+\dfrac{\sqrt{Lt_{\mathrm{mix}}t_{\mathrm{hit}}}}{ T^{(1-\xi)/2}}\right)\\
        &+\dfrac{B}{L}\tilde{\mathcal{O}}\left(\dfrac{AG^2t_{\mathrm{mix}}^2}{\delta T^{\xi}}+\dfrac{Lt_{\mathrm{mix}}t_{\mathrm{hit}}}{\delta T^{1-\xi}}+\delta\beta\right) + \mathcal{\tilde{O}}\bigg(\frac{Lt_{\mathrm{mix}}t_{\mathrm{hit}}\mathbf{E}_{s\sim d^{\pi^*}}[KL(\pi^*(\cdot\vert s)\Vert\pi_{\theta_1}(\cdot\vert s))]}{T^{1-\xi}}\bigg)\\
        &\leq \delta\sqrt{\epsilon_{\mathrm{bias}}} + \tilde{\mathcal{O}}\left(\dfrac{2t_{\mathrm{mix}}t_{\mathrm{hit}}}{\delta \beta T^{1-\xi}}\right) + G\left(1+\dfrac{1}{\mu_F}\right)\tilde{\mathcal{O}}\left(\delta\sqrt{\beta}+\dfrac{\sqrt{A}Gt_{\mathrm{mix}}}{ T^{\xi/2}}+\dfrac{\sqrt{Lt_{\mathrm{mix}}t_{\mathrm{hit}}}}{ T^{(1-\xi)/2}}\right)
    \end{aligned}
\end{equation}

The last inequality presents only the dominant terms of $\beta$ and $T$.

\subsubsection{Optimal Choice of $\beta$ and $\xi$}

If we choose $\beta=T^{-\eta}$ for some $\eta\in(0,1)$, then following \eqref{eq_bound_Jr_final} and \eqref{eq_83}, we can write,
\begin{align}
     \frac{1}{K}\sum_{k=1}^{K}\mathbf{E}\bigg(J_r^{\pi^*}-J_r(\theta_k)\bigg)&\leq  \sqrt{\epsilon_{\mathrm{bias}}} + \tilde{\mathcal{O}}\left(T^{-\eta/2}+T^{-\xi/2}+T^{-(1-\xi)/2}\right),\\
     \mathbf{E}\left[\frac{1}{K}\sum_{k=1}^{K}-J_c(\theta_k)\right]&\leq \delta\sqrt{\epsilon_{\mathrm{bias}}} + \tilde{\mathcal{O}}\left(T^{-(1-\xi-\eta)}+T^{-\eta/2}+T^{-\xi/2}+T^{-(1-\xi)/2}\right)
\end{align}

Clearly, the optimal values of $\eta$ and $\xi$ can be obtained by solving the following optimization.
\begin{align}
    {\max}_{(\eta, \xi)\in (0,1)^2} \min \left\lbrace 1-\xi-\eta, \dfrac{\eta}{2}, \dfrac{\xi}{2}, \dfrac{1-\xi}{2} \right\rbrace 
\end{align}

One can easily verify that $(\xi, \eta) = \left(2/5, 2/5\right)$ is the solution of the above optimization. Therefore, the convergence rate of the objective function can be written as follows.

\begin{equation}
    \label{eq_bound_Jr_final_par_substituted}
    \begin{aligned}
       &\frac{1}{K}\sum_{k=1}^{K}\mathbf{E}\bigg(J_r^{\pi^*}-J_r(\theta_k)\bigg)\\
       &\leq  \sqrt{\epsilon_{\mathrm{bias}}} + G\left(1+\dfrac{1}{\mu_F}\right)\tilde{\mathcal{O}}\left(\dfrac{1}{T^{1/5}}+\dfrac{\sqrt{A}Gt_{\mathrm{mix}}}{\delta T^{1/5}}+\dfrac{\sqrt{Lt_{\mathrm{mix}}t_{\mathrm{hit}}}}{\delta T^{3/10}}\right)+ \mathcal{O}\left(\dfrac{1}{\delta T^2}+\dfrac{1}{T^{2/5}}\right)\\
        &+\dfrac{B}{L}\tilde{\mathcal{O}}\left(\dfrac{\delta^2+AG^2t_{\mathrm{mix}}^2}{\delta^2 T^{2/5}}+\dfrac{Lt_{\mathrm{mix}}t_{\mathrm{hit}}}{\delta^2 T^{3/5}}\right) + \mathcal{\tilde{O}}\bigg(\frac{Lt_{\mathrm{mix}}t_{\mathrm{hit}}\mathbf{E}_{s\sim d^{\pi^*}}[KL(\pi^*(\cdot\vert s)\Vert\pi_{\theta_1}(\cdot\vert s))]}{T^{3/5}\delta }\bigg)\\
        &\leq \sqrt{\epsilon_{\mathrm{bias}}} + \dfrac{\sqrt{A}G^2t_{\mathrm{mix}}}{\delta}\left(1+\dfrac{1}{\mu_F}\right)\tilde{\mathcal{O}}\left(T^{-1/5}\right)
    \end{aligned}
\end{equation}

The last expression only considers the dominant terms of $T$. Similarly, the constraint violation rate can be computed as,
\begin{align}
\label{eq_87_new}
    \begin{split}
        &\mathbf{E}\bigg[\frac{1}{K}\sum_{k=1}^{K}-J_c(\theta_k)\bigg] \\
        &\leq \delta\sqrt{\epsilon_{\mathrm{bias}}}+ \tilde{\mathcal{O}}\left(\dfrac{t_{\mathrm{mix}}t_{\mathrm{hit}}}{\delta T^{1/5}}+\dfrac{1}{ T^2}+\dfrac{\delta}{T^{2/5}}\right) +G\left(1+\dfrac{1}{\mu_F}\right)\tilde{\mathcal{O}}\left(\dfrac{\delta+\sqrt{A}Gt_{\mathrm{mix}}}{ T^{1/5}}+\dfrac{\sqrt{Lt_{\mathrm{mix}}t_{\mathrm{hit}}}}{ T^{3/10}}\right)\\
        &+\dfrac{B}{L}\tilde{\mathcal{O}}\left(\dfrac{\delta^2+ AG^2t_{\mathrm{mix}}^2}{\delta T^{2/5}}+\dfrac{Lt_{\mathrm{mix}}t_{\mathrm{hit}}}{\delta T^{3/5}}\right) + \mathcal{\tilde{O}}\bigg(\frac{Lt_{\mathrm{mix}}t_{\mathrm{hit}}\mathbf{E}_{s\sim d^{\pi^*}}[KL(\pi^*(\cdot\vert s)\Vert\pi_{\theta_1}(\cdot\vert s))]}{T^{3/5}}\bigg)\\
        &\leq \delta\sqrt{\epsilon_{\mathrm{bias}}} + \tilde{\mathcal{O}}\left(\dfrac{t_{\mathrm{mix}}t_{\mathrm{hit}}}{\delta T^{1/5}}\right) + \sqrt{A}G^2t_{\mathrm{mix}}\left(1+\dfrac{1}{\mu_F}\right)\tilde{\mathcal{O}}\left(T^{-1/5}\right)
    \end{split}
\end{align}
where the last expression contains only the dominant terms of $T$. This concludes the theorem.
\section{Proofs for the Regret and Violation Analysis}
\subsection{Proof of Lemma \ref{lemma_last}}
\begin{proof}
    Using Taylor's expansion, we can write the following $\forall (s, a)\in \mathcal{S}\times \mathcal{A}$, $\forall k$.
\begin{align}
\label{eq_pi_lipschitz}
\begin{split}
        |\pi_{\theta_{k+1}}(a|s)-\pi_{\theta_{k}}(a|s)|&=\left|(\theta_{k+1}-\theta_k)^T\nabla_{\theta}\pi_{\bar\theta}(a|s) \right| \\&=\pi_{\bar{\theta}_k}(a|s)\left|(\theta_{k+1}-\theta_k)^T\nabla_{\theta}\log \pi_{\bar{\theta}_k}(a|s) \right|\\
        &\leq \pi_{\bar{\theta}_k}(a|s) \norm{\theta_{k+1}-\theta_k}\norm{\nabla_{\theta}\log \pi_{\bar{\theta}_k}(a|s)}\overset{(a)}{\leq} G\norm{\theta_{k+1}-\theta_k}
\end{split} 
\end{align}
where $\bar{\theta}_k$ is some convex combination\footnote{Note that, in general, $\bar{\theta}_k$ is dependent on $(s, a)$.} of $\theta_{k}$ and $\theta_{k+1}$ and $(a)$ results from Assumption \ref{ass_score}. This concludes the first statement. Applying \eqref{eq_pi_lipschitz} and Lemma \ref{lemma_aux_5}, we obtain the following for $g\in\{r, c\}$.
\begin{align}
\label{eq_long_49}
\begin{split}
		&\sum_{k=1}^{K}\mathbf{E}\big|J_g(\theta_{k+1}) - J_g(\theta_{k})\big| = \sum_{k=1}^{K}\mathbf{E}\left|\sum_{s,a}d^{\pi_{\theta_{k+1}}}(s)(\pi_{\theta_{k+1}}(a|s)-\pi_{\theta_{k}}(a|s))Q_g^{\pi_{\theta_{k}}}(s, a)\right|\\
        &\leq \sum_{k=1}^{K}\mathbf{E}\left[\sum_{s,a}d^{\pi_{\theta_{k+1}}}(s)\left|\pi_{\theta_{k+1}}(a|s)-\pi_{\theta_{k}}(a|s)\right|\left|Q_g^{\pi_{\theta_{k}}}(s, a)\right|\right]\\
		&\leq G\sum_{k=1}^{K}\mathbf{E}\left[\sum_{s,a}d^{\pi_{\theta_{k+1}}}(s)\Vert\theta_{k+1}-\theta_{k}\Vert|Q_g^{\pi_{\theta_{k}}}(s, a)|\right]\\
        &\overset{(a)}{\leq} G\alpha\sum_{k=1}^{K}\mathbf{E}\left[\sum_{a}\underbrace{\sum_{s}d^{\pi_{\theta_{k+1}}}(s)}_{=1}\Vert\omega_k\Vert\cdot 6t_{\mathrm{mix}}\right]
        \overset{}{=} 6AG\alpha t_{\mathrm{mix}} \sum_{k=1}^{K}\mathbf{E}\norm{\omega_k}\\
        &\overset{(b)}\leq 6AG\alpha t_{\mathrm{mix}}\sqrt{K}\left(\sum_{k=1}^{K}\mathbf{E}\norm{\omega_k}^2\right)^{\frac{1}{2}}\\
        &\overset{(c)}{\leq} \mathcal{\tilde{O}}\left(\dfrac{\alpha AG}{\delta t_{\mathrm{hit}}}\left[\left(\sqrt{A}G t_{\mathrm{mix}}+\delta\right)T^{\frac{2}{5}}+\sqrt{Lt_{\mathrm{mix}}t_{\mathrm{hit}}}T^{\frac{3}{10}}\right]\right)
        \end{split}
	\end{align}
 
Inequality $(a)$ uses Lemma \ref{lemma_aux_2} and the update rule $\theta_{k+1}=\theta_k+\alpha \omega _k$. Step $(b)$ holds by the Cauchy inequality and Jensen inequality whereas $(c)$ can be derived using $\eqref{eq_34}$ and substituting $K=T/H$. This establishes the second statement. Next, recall from $(\ref{eq_r_pi_theta})$ that for any policy $\pi_{\theta}$,  $g^{\pi_{\theta}}(s) \triangleq \sum_a\pi_{\theta}(a|s)g(s, a)$.
	Note that, for any policy parameter $\theta$, and any state $s\in\mathcal{S}$, the following holds.
	\begin{align}
        \label{eq_49}
		V_g^{\pi_{\theta}}(s)=\sum_{t=0}^{\infty}\left<(P^{\pi_{\theta}})^t(s,\cdot) - d^{\pi_{\theta}}, g^{\pi_{\theta}}\right> = \sum_{t=0}^{N-1}\left<(P^{\pi_{\theta}})^t({s,\cdot}),g^{\pi_{\theta}}\right> - NJ(\theta) + \sum_{t=N}^{\infty}\left<(P^{\pi_{\theta}})^t({s,\cdot}) - d^{\pi_{\theta}},g^{\pi_{\theta}}\right>.
	\end{align}

    Define the following quantity.
    \begin{align}
        \label{def_error}
        \delta^{\pi_{\theta}}(s, T) \triangleq \sum_{t=N}^{\infty}\norm{(P^{\pi_{\theta}})^t({s,\cdot}) - d^{\pi_{\theta}}}_1 ~~\text{where} ~N=4t_{\mathrm{mix}}(\log_2 T)
    \end{align}
 
    Lemma \ref{lemma_aux_3} states that for sufficiently large $T$, we have $\delta^{\pi_{\theta}}(s, T)\leq \frac{1}{T^3}$ for any policy $\pi_{\theta}$ and state $s$. Combining this result with the fact that the $g^{\pi_\theta}$ function is absolutely bounded in $[0, 1]$, we obtain,
    \begin{align}
        \label{eq_exp_diff_v}
        \begin{split}    
            &\sum_{k=1}^K\mathbf{E}|V_g^{\pi_{\theta_{k+1}}}(s_k) - V_g^{\pi_{\theta_{k}}}(s_k)|\\
            &\leq \sum_{k=1}^K\mathbf{E}\left|\sum_{t=0}^{N-1}\left<(P^{\pi_{\theta_{k+1}}})^t({s_k,\cdot}) - (P^{\pi_{\theta_k}})^t({s_k,\cdot}), g^{\pi_{\theta_{k+1}}}\right>\right| + \sum_{k=1}^K\mathbf{E}\left|\sum_{t=0}^{N-1}\left<(P^{\pi_{\theta_k}})^t({s_k,\cdot}), g^{\pi_{\theta_{k+1}}}-g^{\pi_{\theta_k}}\right>\right| \\
            &+ N\sum_{k=1}^K\mathbf{E}|J_g(\theta_{k+1}) - J_g(\theta_{k})| + \frac{2K}{T^3}\\
            &\overset{(a)}{\leq} \sum_{k=1}^K\sum_{t=0}^{N-1}\mathbf{E}\norm{ (P^{\pi_{\theta_{k+1}}})^t - (P^{\pi_{\theta_k}})^t)g^{\pi_{\theta_{k+1}}} }_{\infty} + \sum_{k=1}^K\sum_{t=0}^{N-1}\mathbf{E}\norm{g^{\pi_{\theta_{k+1}}}-g^{\pi_{\theta_k}}}_{\infty} \\
            &+ \mathcal{\tilde{O}}\left(\dfrac{\alpha AG t_{\mathrm{mix}}}{\delta t_{\mathrm{hit}}}\left[\left(\sqrt{A}G t_{\mathrm{mix}}+\delta\right)T^{\frac{2}{5}}+\sqrt{Lt_{\mathrm{mix}}t_{\mathrm{hit}}}T^{\frac{3}{10}}\right]\right)
        \end{split}
    \end{align}
    where $(a)$ follows from \eqref{eq_long_49} and substituting $N=4t_{\mathrm{mix}}(\log_2 T)$. For the first term, note that,
\begin{align}
    \label{eq_long_recursion}
    \begin{split}
        &\norm{ ((P^{\pi_{\theta_{k+1}}})^t - (P^{\pi_{\theta_k}})^t)g^{\pi_{\theta_{k+1}}} }_{\infty}\\ &\leq \norm{ P^{\pi_{\theta_{k+1}}}((P^{\pi_{\theta_{k+1}}})^{t-1} - (P^{\pi_{\theta_k}})^{t-1})g^{\pi_{\theta_{k+1}}} }_{\infty} + \norm{ (P^{\pi_{\theta_{k+1}}} - P^{\pi_{\theta_k}})(P^{\pi_{\theta_k}})^{t-1}g^{\pi_{\theta_{k+1}}} }_{\infty}\\
        &\overset{(a)}{\leq} \norm{ ((P^{\pi_{\theta_{k+1}}})^{t-1} - (P^{\pi_{\theta_k}})^{t-1})g^{\pi_{\theta_{k+1}}} }_{\infty} + \max_s\norm{P^{\pi_{\theta_{k+1}}}({s,\cdot})-P^{\pi_{\theta_k}}({s,\cdot})}_1
    \end{split}
\end{align}

Inequality $(a)$ holds since every row of $P^{\pi_{\theta_k}}$ sums to $1$ and $\norm{(P^{\pi_{\theta_k}})^{t-1}g^{\pi_{\theta_{k+1}}}}_{\infty}\leq 1$. Moreover, invoking \eqref{eq_pi_lipschitz}, and the parameter update rule $\theta_{k+1}=\theta_k + \alpha \omega_k$, we get,
\begin{align*}
    \max_s\Vert P^{\pi_{\theta_{k+1}}}({s,\cdot})-P^{\pi_{\theta_k}}({s,\cdot})\Vert_1 &= \max_s\left| \sum_{s'}\sum_a(\pi_{\theta_{k+1}}(a|s)-\pi_{\theta_k}(a|s))P(s'|s, a)\right| \\
    &\leq G \norm{\theta_{k+1}-\theta_k}\max_s\left| \sum_{s'}\sum_a P(s'|s, a)\right| \\
    &\leq \alpha A G\norm{\omega_k}  
\end{align*}
 
Plugging the above result into $(\ref{eq_long_recursion})$ and using a recursive argument, we get,
\begin{align*}
    \norm{ ((P^{\pi_{\theta_{k+1}}})^t - (P^{\pi_{\theta_k}})^t)g^{\pi_{\theta_{k+1}}} }_{\infty} &\leq \sum_{t'=1}^{t} \max_s\norm{P^{\pi_{\theta_{k+1}}}({s,\cdot})-P^{\pi_{\theta_k}}({s,\cdot})}_1\\
    &\leq \sum_{t'=1}^{t}\alpha AG\norm{\omega_k}  \leq \alpha t A G\norm{\omega_k}
\end{align*}
    
Finally, we have
\begin{align}
    \label{eq_app_54}
    \begin{split}
        \sum_{k=1}^K\sum_{t=0}^{N-1} &\mathbf{E}\norm{ ((P^{\pi_{\theta_{k+1}}})^t - (P^{\pi_{\theta_k}})^t)g^{\pi_{\theta_{k+1}}} }_{\infty}\\
        &\leq \sum_{k=1}^K\sum_{t=0}^{N-1}\alpha t A G\norm{\omega_k}\\
        &\leq \mathcal{O}(\alpha A G N^2) \sum_{k=1}^K\mathbf{E}\norm{\omega_k}\\
        &\leq \mathcal{O}(\alpha A G N^2\sqrt{K}) \left(\sum_{k=1}^K\mathbf{E}\norm{\omega_k}^2\right)^{\frac{1}{2}}\\
        &\overset{(a)}{=} \mathcal{\tilde{O}}\left(\dfrac{\alpha AG t_{\mathrm{mix}}}{\delta t_{\mathrm{hit}}}\left[\left(\sqrt{A}G t_{\mathrm{mix}}+\delta\right)T^{\frac{2}{5}}+\sqrt{Lt_{\mathrm{mix}}t_{\mathrm{hit}}}T^{\frac{3}{10}}\right]\right)
    \end{split}
\end{align}
where $(a)$ follows from \eqref{eq_34}. Moreover, notice that,
\begin{equation}
    \label{eq_app_55}
    \begin{aligned}
        \sum_{k=1}^{K}\sum_{t=0}^{N-1}\mathbf{E}\norm{g^{\pi_{\theta_{k+1}}}-g^{\pi_{\theta_{k}}}}_{\infty}&\overset{}{\leq} \sum_{k=1}^{K}\sum_{t=0}^{N-1}\mathbf{E}\left[\max_s\left|\sum_a(\pi_{\theta_{k+1}}(a|s)-\pi_{\theta_{k}}(a|s))g(s,a)\right|\right]\\
        &\overset{(a)}{\leq}\alpha AGN \sum_{k=1}^{K} \mathbf{E}\norm{\omega_k}\\
        &\leq \alpha AGN\sqrt{K} \left(\sum_{k=1}^{K} \mathbf{E}\norm{\omega_k}^2\right)^{\frac{1}{2}}\\
        &\overset{(b)}{\leq} \mathcal{\tilde{O}}\left(\dfrac{\alpha AG}{\delta t_{\mathrm{hit}}}\left[\left(\sqrt{A}G t_{\mathrm{mix}}+\delta\right)T^{\frac{2}{5}}+\sqrt{Lt_{\mathrm{mix}}t_{\mathrm{hit}}}T^{\frac{3}{10}}\right]\right)
    \end{aligned}
\end{equation}
where $(a)$ follows from \eqref{eq_pi_lipschitz} and the update rule $\theta_{k+1}=\theta_k + \alpha \omega_k$ whereas $(b)$ is a consequence of \eqref{eq_34}. Combining \eqref{eq_exp_diff_v}, \eqref{eq_app_54}, and  \eqref{eq_app_55}, we establish the third statement.
\end{proof}

\subsection{Proof of Theorem \ref{thm_regret}}
\begin{proof}
Recall the decomposition of the regret in section \ref{sec_regret} and take the expectation.
    \begin{align}
        \begin{split}
        &\mathbf{E}[\mathrm{Reg}_T] = \sum_{t=0}^{T-1} \left(J_r^{\pi^*} - r(s_t, a_t)\right)=H\sum_{k=1}^{K}\left(J_r^{\pi^*}-J_r({\theta_k})\right)+\sum_{k=1}^{K}\sum_{t\in\mathcal{I}_k} \left(J_r(\theta_k)-r(s_t, a_t)\right)\\
        &=H\sum_{k=1}^{K}\left(J_r^{\pi^*}-J_r({\theta_k})\right)+\mathbf{E}\left[\sum_{k=1}^{K-1} V_r^{\pi_{\theta_{k+1}}}(s_{kH})-V_r^{\pi_{\theta_k}}(s_{kH})\right]+\mathbf{E}\left[ V_r^{\pi_{\theta_K}}(s_{T})-V_r^{\pi_{\theta_0}}(s_{0})\right]
        \end{split}
    \end{align}
    Using the result in \eqref{eq_bound_Jr_final_par_substituted}, Lemma \ref{lemma_last} and Lemma \ref{lemma_aux_2}, we get,
    \begin{align}
        \begin{split}
        &\mathbf{E}[\mathrm{Reg}_T]\leq T\sqrt{\epsilon_{\mathrm{bias}}}+G\left(1+\dfrac{1}{\mu_F}\right)\tilde{\mathcal{O}}\left(T^{\frac{4}{5}}+\dfrac{\sqrt{A}Gt_{\mathrm{mix}}}{\delta }T^{\frac{4}{5}}+\dfrac{\sqrt{Lt_{\mathrm{mix}}t_{\mathrm{hit}}}}{\delta}T^{\frac{7}{10}}\right)+ \mathcal{O}\left(\dfrac{1}{T}+T^{\frac{3}{5}}\right)\\
        &+\dfrac{B}{L}\tilde{\mathcal{O}}\left(\dfrac{\delta^2+AG^2t_{\mathrm{mix}}^2}{\delta^2}T^{\frac{3}{5}}+\dfrac{Lt_{\mathrm{mix}}t_{\mathrm{hit}}}{\delta^2}T^{\frac{2}{5}}\right)
        +\mathcal{\tilde{O}}\bigg(\frac{Lt_{\mathrm{mix}}t_{\mathrm{hit}}\mathbf{E}_{s\sim d^{\pi^*}}[KL(\pi^*(\cdot\vert s)\Vert\pi_{\theta_1}(\cdot\vert s))]}{\delta }T^{\frac{2}{5}}\bigg) \\
        &+ \mathcal{\tilde{O}}\left(\dfrac{\alpha AG t_{\mathrm{mix}}}{\delta t_{\mathrm{hit}}}\left[\left(\sqrt{A}G t_{\mathrm{mix}}+\delta\right)T^{\frac{2}{5}}+\sqrt{Lt_{\mathrm{mix}}t_{\mathrm{hit}}}T^{\frac{3}{10}}\right]\right)+\mathcal{O}(t_{\mathrm{mix}})
        \end{split}
    \end{align}
Similarly, for the constraint violation, we have
    \begin{align}
        \begin{split}
        \mathbf{E}[\mathrm{Vio}_T] &= \sum_{t=0}^{T-1} \left(- c(s_t, a_t)\right)=H\sum_{k=1}^{K}-J_c({\theta_k})+\sum_{k=1}^{K}\sum_{t\in\mathcal{I}_k} \left(J_c(\theta_k)-c(s_t, a_t)\right)\\
        &=-H\sum_{k=1}^{K}J_c({\theta_k})+\mathbf{E}\left[\sum_{k=1}^{K-1} V_c^{\pi_{\theta_{k+1}}}(s_{kH})-V_c^{\pi_{\theta_k}}(s_{kH})\right]+\mathbf{E}\left[ V_c^{\pi_{\theta_K}}(s_{T})-V_c^{\pi_{\theta_0}}(s_{0})\right]
        \end{split}
    \end{align}
    Using the result in \eqref{eq_87_new}, Lemma \ref{lemma_last} and Lemma \ref{lemma_aux_2}, we get,
    \begin{align}
        \begin{split}
        \mathbf{E}[\mathrm{Vio}_T]&\leq T\delta\sqrt{\epsilon_{\mathrm{bias}}}+G\left(1+\dfrac{1}{\mu_F}\right)\tilde{\mathcal{O}}\left(\left[\delta+\sqrt{A}Gt_{\mathrm{mix}}\right]T^{\frac{4}{5}} + \sqrt{Lt_{\mathrm{mix}}t_{\mathrm{hit}}}T^{\frac{7}{10}}\right)\\
        &+ \mathcal{O}\left(\dfrac{t_{\mathrm{mix}}t_{\mathrm{hit}}}{\delta}T^{\frac{4}{5}}+\frac{1}{\delta T}+\delta T^{\frac{3}{5}}\right)+\dfrac{B}{L}\tilde{\mathcal{O}}\left(\dfrac{\delta^2+AG^2t_{\mathrm{mix}}^2}{\delta}T^{\frac{3}{5}}+\dfrac{Lt_{\mathrm{mix}}t_{\mathrm{hit}}}{\delta}T^{\frac{2}{5}}\right)\\
        &+\mathcal{\tilde{O}}\bigg(Lt_{\mathrm{mix}}t_{\mathrm{hit}}\mathbf{E}_{s\sim d^{\pi^*}}[KL(\pi^*(\cdot\vert s)\Vert\pi_{\theta_1}(\cdot\vert s))] T^{\frac{2}{5}}\bigg) \\
        &+ \mathcal{\tilde{O}}\left(\dfrac{\alpha AG t_{\mathrm{mix}}}{\delta t_{\mathrm{hit}}}\left[\left(\sqrt{A}G t_{\mathrm{mix}}+\delta\right)T^{\frac{2}{5}}+\sqrt{Lt_{\mathrm{mix}}t_{\mathrm{hit}}}T^{\frac{3}{10}}\right]\right)+\mathcal{O}(t_{\mathrm{mix}})
        \end{split}
    \end{align}

    This concludes the theorem.
\end{proof}

\section{Some Auxiliary Lemmas for the Proofs}
\begin{lemma}
    \label{lemma_aux_2}
    \cite[Lemma 14]{wei2020model} For any ergodic MDP with mixing time $t_{\mathrm{mix}}$, the following holds $\forall (s, a)\in\mathcal{S}\times \mathcal{A}$, any policy $\pi$ and $\forall g\in\{r, c\}$.
    \begin{align*}
        (a) |V_g^{\pi}(s)|\leq 5 t_{\mathrm{mix}},~~
        (b) |Q_g^{\pi}(s, a)|\leq 6 t_{\mathrm{mix}}
    \end{align*}
\end{lemma}

\begin{lemma}
\label{lemma_aux_3}
    \cite[Corollary 13.2]{wei2020model} Let $\delta^{\pi}(\cdot, T)$ be defined as written below for an arbitrary policy $\pi$. 
    \begin{align}
        \label{def_error_aux_1}
        \delta^{\pi}(s, T) \triangleq \sum_{t=N}^{\infty}\norm{(P^{\pi})^t({s,\cdot}) - d^{\pi}}_1, ~\forall s\in\mathcal{S} ~~\text{where} ~N=4t_{\mathrm{mix}}(\log_2 T)
    \end{align}
    
    If $t_{\mathrm{mix}}<T/4$, we have the following inequality $\forall s\in\mathcal{S}$: $\delta^{\pi}(s, T)\leq \frac{1}{T^3}$.
\end{lemma}

\begin{lemma}
\label{lemma_aux_4}
    \cite[Lemma 16]{wei2020model} Let $\mathcal{I}=\{t_1+1,t_1+2,\cdots,t_2\}$ be a certain period of an epoch $k$ of Algorithm \ref{alg:estQ} with length $N$. Then for any $s$, the probability that the algorithm never visits $s$ in $\mathcal{I}$ is upper bounded by
    \begin{equation}
        \left(1-\frac{3d^{\pi_{\theta_k}}(s)}{4}\right)^{\left\lfloor\frac{\lfloor \mathcal{I}\rfloor}{N}\right\rfloor}
    \end{equation}
\end{lemma}

\begin{lemma}
\label{lemma_aux_5}
    \cite[Lemma 15]{wei2020model} The difference of the values of the function $J_g$, $g\in\{r, c\}$ at policies $\pi$ and $\pi'$, is
    \begin{equation}
        J_g^{\pi}-J_g^{\pi'}=\sum_{s}\sum_{a}d^{\pi}(s)(\pi(a|s)-\pi'(a|s))Q_g^{\pi'}(s,a)
    \end{equation}
\end{lemma}

\begin{lemma}
\label{lemma_aux_6}
    \cite[Lemma 7]{chen2022learning} The term $\hat{J}_c(\theta)$ for any $\theta\in\Theta$ is a good estimator of $J_c(\theta)$, which means
    \begin{equation}
        \big|\mathbf{E}[\hat{J}_c(\theta)]-J_c(\theta)\big|\leq \frac{1}{T^2}
    \end{equation}
\end{lemma}

\begin{lemma}
\label{lemma_aux_7}
    \cite[Lemma A.6]{dorfman2022adapting} Let  $\theta\in\Theta$ be a policy parameter. Fix a trajectory $z=\{(s_t, a_t, r_t, s_{t+1})\}_{t\in\mathbb{N}}$ generated by following the policy $\pi_{\theta}$ starting from some initial state $s_0\sim\rho$. Let, $\nabla L(\theta)$ be the gradient that we wish to estimate over $z$, and $l(\theta, \cdot)$ is a function such that $\mathbf{E}_{z\sim d^{\pi_{\theta}}, \pi_{\theta}}l(\theta, z)=\nabla L(\theta)$. Assume that $\norm{l(\theta, z)}, \norm{\nabla L(\theta)}\leq G_L$, $\forall \theta\in\Theta$, $\forall z\in \mathcal{S}\times \mathcal{A}\times \mathbb{R}\times \mathcal{S}$. Define $l^{Q}=\frac{1}{Q}\sum_{i=1}^Q l(\theta, z_i)$. If $P=2t_{\mathrm{mix}}\log T$, then the following holds as long as $Q\leq T$,
    \begin{align}
        \mathbf{E}\left[\norm{l^{Q}-\nabla L(\theta)}^2\right]\leq \mathcal{O}\left(G_L^2\log\left(PQ\right)\dfrac{P}{Q}\right)
    \end{align}
\end{lemma}

\begin{lemma}[Strong duality]\cite[Lemma 3]{ding2023convergence}
    \label{lem.duality}
    For convenience, we rewrite the unparameterized problem \eqref{eq:def_unparametrized_optimization}.
\begin{equation}\label{eq:rewrite_unparameterized}
    \begin{aligned}
        \max_{\pi\in\Pi} ~& J_r^{\pi} \\
        \text{s.t.} ~& J_c^{\pi}\geq 0
    \end{aligned}
\end{equation} 

Define $\pi^*$ as the optimal solution to the above problem. Define the associated dual function as
\begin{equation}
    J_D^{\lambda}\triangleq\max_{\pi\in\Pi} J_r^{\pi}+\lambda J_c^{\pi}
\end{equation}
and denote $\lambda^*=\arg\min_{\lambda\geq 0} J_D^{\lambda}$. We have the following strong duality property for the unparameterized problem whenever Assumption \ref{ass_slater} holds.
    \begin{equation}\label{eq:duality}
        J_r^{\pi^*} = J_D^{\lambda^*} 
    \end{equation}	
\end{lemma}

Although the strong duality holds for the unparameterized problem, the same is not true for parameterized class $\{\pi_\theta|\theta\in \Theta\}$. To formalize this statement, define the dual function associated with the parameterized problem as follows.
\begin{equation}
    J_{D,\Theta}^{\lambda}\triangleq\max_{\theta\in \Theta}  J_r(\theta)+\lambda J_c(\theta)
\end{equation}
and denote $\lambda_\Theta^*=\arg\min_{\lambda\geq 0} J_{D,\Theta}^{\lambda}$. The lack of strong duality states that, in general, $J_{D, \Theta}^{\lambda_{\Theta}^*}\neq J_r(\theta^*)$ where $\theta^*$ is a solution of the parameterized constrained optimization \eqref{eq:def_constrained_optimization}. However, the parameter $\lambda_\Theta^*$, as we demonstrate below, must obey some restrictions.
\begin{lemma}
	\label{lem.boundness}
	Under  Assumption \ref{ass_slater}, the optimal dual variable for the parameterized problem is bounded as
    \begin{equation}
        0 \leq \lambda_\Theta^* \leq \frac{J_r^{\pi^*}-J_r(\bar{\theta})}{\delta}\leq \dfrac{1}{\delta}
    \end{equation}
\end{lemma}

\begin{proof}
    The proof follows the approach in \cite[Lemma 3]{ding2023convergence}, but is revised to the general parameterization setup.	Let $\Lambda_a\triangleq\{ \lambda\geq 0\,\vert\, J_{D,\Theta}^\lambda \leq a \}$ be a sublevel set of the dual function for $a\in\mathbb{R}$. If $\Lambda_a$ is non-empty, then for any $\lambda \in\Lambda_a$, 
	\begin{equation}
	    a\geq J_{D,\Theta}^\lambda\geq J_r(\bar{\theta})+\lambda J_c(\bar{\theta})\geq J_r(\bar{\theta})+\lambda \delta
	\end{equation}
    where $\bar{\theta}$ is a Slater point in Assumption \ref{ass_slater}. Thus, $\lambda \leq (a -J_r(\bar{\theta}))/\delta$.	If we take $a= J_{D,\Theta}^{\lambda_\Theta^*}\leq J_{D,\Theta}^{\lambda^*} \leq J_D^{\lambda^*}=J_r^{\pi^*}$, then we have $\lambda_\Theta^*\in \Lambda_a$, which proves the Lemma. The last inequality holds since $J_r^{\pi}\in [0,1]$ for any policy, $\pi$.
\end{proof}

    Since the above inequality holds for arbitrary $\Theta$, we also have, $0\leq \lambda^*\leq \frac{1}{\delta}$. Define $v(\tau)\triangleq\max_{\pi\in\Pi}\{J_r^\pi|J_c^\pi\geq \tau\}$. Using the strong duality property of the unparameterized problem \eqref{eq:rewrite_unparameterized}, we establish the following property of the function, $v(\cdot)$.

\begin{lemma}
    \label{lem:bridge}
    Assume that the Assumption \ref{ass_slater} holds, we have for any $\tau\in\mathbb{R}$,
    \begin{equation}
        v(0)-\tau\lambda^* \geq	v(\tau)
    \end{equation}
\end{lemma}

\begin{proof}
    By the definition of $v(\tau)$, we have $v(0) = J_r^{\pi^*}$. With a slight abuse of notation, denote $J_{\mathrm{L}}(\pi,\lambda)=J_r^{\pi}+\lambda J_c^{\pi}$. By the strong duality stated in Lemma \ref{lem.duality}, we have the following for any $\pi\in\Pi$.
    \begin{equation}
        J_{\mathrm{L}}(\pi,\lambda^*)\leq \max_{\pi\in\Pi} J_{\mathrm{L}}(\pi,\lambda^*)\overset{Def}=J_D^{\lambda^*}\overset{\eqref{eq:duality}}=J_r^{\pi^*}=v(0)
    \end{equation}
    Thus, for any $\pi\in\{ \pi\in\Pi \,\vert\,J_c^{\pi} \geq \tau \}$,
    \begin{equation}
	\begin{aligned}
            v(0)-\tau\lambda^*&\geq J_{\mathrm{L}}(\pi,\lambda^*)-\tau\lambda^*\\
            &=J_r^{\pi}+\lambda^*(J_c^\pi-\tau) \geq J_r^{\pi}
	\end{aligned}
    \end{equation}
    Maximizing the right-hand side of this inequality over $\{ \pi\in\Pi \vert J_{c}^{\pi}\geq \tau \}$ yields
    \begin{equation}
        \label{eq.opt1}
	v(0)- \tau\lambda^* \geq v(\tau)
    \end{equation}
    This completes the proof of the lemma.
\end{proof}

We note that a similar result was shown in \cite[Lemma 15]{bai2023provably}. However, the setup of the stated paper is different from that of ours. Specifically, \cite{bai2023provably} considers a tabular setup with peak constraints. Note that Lemma \ref{lem:bridge} has no direct connection with the parameterized setup since its proof uses strong duality and the function, $v(\cdot)$, is defined via a constrained optimization over the entire policy set, $\Pi$, rather than the parameterized policy set. Interestingly, however, the relationship between $v(\tau)$ and $v(0)$ leads to the lemma stated below which turns out to be pivotal in establishing regret and constraint violation bounds in the parameterized setup.

\begin{lemma}\label{lem.constraint}
    Let Assumption \ref{ass_slater} hold. For any constant $C\geq 2\lambda^*$, if there exists a $\pi\in\Pi$ and $\zeta>0$ such that $J_r^{\pi^*}-J_r^{\pi}+C[-J_c^{\pi}]\leq \zeta$, then 
    \begin{equation}
        -J_c^{\pi}\leq 2\zeta/C
    \end{equation}
\end{lemma}
\begin{proof}
 Let $\tau = J^{\pi}_c$. Using the definition of $v(\tau)$, one can write,
	\begin{equation}\label{eq.opt2}
		J_r^{\pi}\leq v(\tau)
	\end{equation}
    Combining Eq. \eqref{eq.opt1} and \eqref{eq.opt2}, we obtain the following.
    \begin{equation}
        J_r^{\pi}-J_r^{\pi^*}\leq v(\tau)-v(0)\leq -\tau\lambda^*
    \end{equation}
	The condition in the Lemma leads to,
	\begin{equation}
	    (C - \lambda^*) (-\tau) = 
		{\tau} \lambda^*+C (-\tau)
        \leq 
		J_r^{\pi^*}-J_r^{\pi}+C [-J_c^{\pi}]\leq \zeta
	\end{equation}
    Finally, we have,
    \begin{equation}
        -\tau\leq \frac{\zeta}{C-\lambda^*}\leq\frac{2\zeta}{C}
    \end{equation}
	which completes the proof.
\end{proof}

\end{document}